\documentclass{article}

 \PassOptionsToPackage{numbers, sort&compress}{natbib}

\usepackage[final]{nips_2018}

\usepackage[utf8]{inputenc} 
\usepackage[T1]{fontenc}    
\usepackage{hyperref}       
\usepackage{url}            
\usepackage{booktabs}       
\usepackage{amsfonts}       
\usepackage{nicefrac}       
\usepackage{microtype}      
\usepackage[caption=false,font=normalsize]{subfig} 

\usepackage{bm}
\usepackage{bbm}
\usepackage{amsmath}
\usepackage{amssymb} 
\usepackage{xcolor}
\usepackage{amsthm}
\usepackage{thmtools}
\usepackage{thm-restate}
\usepackage{mathtools}
\usepackage{algorithm}
\usepackage[noend]{algpseudocode}
\usepackage{appendix}
\usepackage{enumitem}

\usepackage{accents}
\newcommand{\ubar}[1]{\underaccent{\bar}{#1}}
\DeclarePairedDelimiter{\ceil}{\lceil}{\rceil}
\DeclarePairedDelimiter\floor{\lfloor}{\rfloor}
\newcommand{\defeq}{\vcentcolon=}

\DeclareMathOperator*{\argmin}{arg\,min}
\DeclareMathOperator*{\argmax}{arg\,max}
\newcommand{\Var}{{\rm Var}}
\newcommand{\Z}{\mathbb{Z}}

\newcommand{\bOne}{\mathbbm{1}}

\newcommand{\bI}{\boldsymbol{I}}

\newcommand{\cE}{\mathcal{E}}

\newcommand{\cS}{\mathcal{S}}

\newcommand{\abs}[1]{\left| #1 \right|}
\newcommand{\norm}[1]{\| #1 \|}
\newcommand{\dom}{{\rm dom}}

\definecolor{red}{HTML}{E51400} 
\definecolor{blue}{HTML}{0050EF} 
\definecolor{green}{HTML}{008A00} 
\definecolor{purple}{HTML}{AA00FF} 
\definecolor{orange}{HTML}{FF7F00}
\definecolor{gray}{HTML}{848482}

\hypersetup{
    colorlinks=true,
    linkcolor=blue,
    filecolor=magenta,      
    urlcolor=blue,
}

\algrenewcommand\algorithmicrequire{\textbf{Input}}
\algrenewcommand\algorithmicensure{\textbf{Output}}
\renewcommand{\algorithmiccomment}[1]{\bgroup\hfill\small$\triangleright${\color{black!80}\textit{ #1}}\egroup}
\algdef{SE}[PROCEDURE]{Procedure}{EndProcedure}%
   [2]{\algorithmicprocedure\ \textproc{#1}\ifthenelse{\equal{#2}{}}{}{(#2)}}%
   {\algorithmicend\ \algorithmicprocedure}%


\newtheorem{theorem}{\bf Theorem}
\newtheorem{lemma}{\bf Lemma}

\newtheorem{property}{\bf Property}
\newtheorem{corollary}{\bf Corollary}
\newtheorem{definition}{\bf Definition}

\newcommand{\compilefullversion}{false}

\ifthenelse{\equal{\compilefullversion}{true}}{%
	\newcommand{\OnlyInFull}[1]{}
	\newcommand{\OnlyInShort}[1]{#1}
}{%
	\newcommand{\OnlyInFull}[1]{#1}%
	\newcommand{\OnlyInShort}[1]{}%
}%

\title{Community Exploration: From Offline Optimization to Online Learning}

\author{
  Xiaowei Chen$^1$, Weiran Huang$^2$, Wei Chen$^3$, John C.S. Lui$^1$ \\
  $^1$The Chinese University of Hong Kong\\
  $^2$Huawei Noah's Ark Lab, $^3$Microsoft Research\\
  \texttt{$^1$\{xwchen, cslui\}@cse.cuhk.edu.hk, $^2$huang.inbox@outlook.com}\\
  \texttt{$^3$weic@microsoft.com}
}

\begin{document}

\maketitle

\begin{abstract}
We introduce the community exploration problem that has many real-world
applications such as online advertising. In the problem, an explorer allocates
limited budget to explore communities so as to maximize the number of members he
could meet. We provide a systematic study of the community exploration problem,
from offline optimization to online learning. For the offline setting where the
sizes of communities are known, we prove that the greedy methods for both of
non-adaptive exploration and adaptive exploration are optimal. For the online
setting where the sizes of communities are not known and need to be learned from
the multi-round explorations, we propose an ``upper confidence'' like algorithm
that achieves the logarithmic regret bounds. By combining the feedback from
different rounds, we can achieve a constant regret bound.
\end{abstract}

\section{Introduction}
In this paper, we introduce the community exploration problem, which is
abstracted from many real-world applications. Consider the following
hypothetical scenario. Suppose that John just entered the university as a
freshman. He wants to explore different student communities or study groups at the
university to meet as many new friends as possible. But he only has a limited
time to spend on exploring different communities, so his problem is how to
allocate his time and energy to explore different student communities to
maximize the number of people he would meet.

The above hypothetical community exploration scenario can also find similar
counterparts in serious business and social applications. One example is online
advertising. In this application, an advertiser wants to promote his products
via placing advertisements on different online websites. The website would show
the advertisements on webpages, and visitors to the websites may click on the
advertisements when they view these webpages. The advertiser wants to reach as
many unique customers as possible, but he only has a limited budget to spend.
Moreover, website visitors come randomly, so it is not guaranteed that all
visitors to the same website are unique customers. So the advertiser needs to
decide how to spend the budget on each website to reach his customers. Of
course, intuitively he should spend more budget on larger communities, but how
much? And what if he does not know the user size of every website? In this
case, each website is a community, consisting of all visitors to this website,
and the problem can be modeled as a community exploration problem. Another example could
be a social worker who wants to reach a large number of people from different
communities to do social studies or improve the social welfare for a large
population, while he also needs to face the budget constraint and uncertainty
about the community.

In this paper, we abstract the common features of these applications and define
the following community exploration problem that reflects the common
core of the problem. We model the problem with $m$ disjoint communities $C_1,
\dots, C_m$ with $C=\cup_{i=1}^m C_i$, where each community $C_i$ has $d_i$
members. Each time when one explores (or visit) a community $C_i$, he would
meet one member of $C_i$ uniformly at random.\footnote{The model can be extended
to meet multiple members per visit, but for simplicity, we consider meeting one
member per visit in this paper.} Given a budget $K$, the goal of community
exploration is to determine the budget allocation $\bm k=(k_1, \dots, k_m)\in
\Z_+^m$ with $\sum_{i=1}^m k_i \le K$, such that the total number of distinct
members met is maximized when each community $C_i$ is explored $k_i$ times.

We provide a systematic study of the above community exploration problem, from
offline optimization to online learning. First, we consider the offline setting
where the community sizes are known. In this setting, we further study two
problem variants --- the non-adaptive version and the adaptive version. The non-adaptive
version requires that the complete budget allocation $\bm k$ is decided before the
exploration is started, while the adaptive version allows the algorithm to use
the feedback from the exploration results of the previous steps to determine the
exploration target of the next step. In both cases, we prove that the greedy
algorithm provides the optimal solution. While the proof for the non-adaptive
case is simple, the proof that the adaptive greedy policy is optimal is much
more involved and relies on a careful analysis of transitions between system
statuses. The proof techniques may be applicable in the analysis of other related
problems.

Second, we consider the online setting where the community sizes are unknown in
advance, which models the uncertainty about the communities in real
applications. We apply the multi-armed bandit (MAB) framework to this task, in
which community explorations proceed in multiple rounds, and in each round we
explore communities with a budget of $K$, use the feedback to learn about the
community size, and adjust the exploration strategy in future rounds. The reward
of a round is the the expected number of unique people met in the round. The
goal is to maximize the cumulative reward from all rounds, or minimizing the
regret, which is defined as the difference in cumulative reward between always
using the optimal offline algorithm when knowing the community sizes and using
the online learning algorithm. Similar to the offline case, we also consider
the non-adaptive and adaptive version of exploration within each round. We provide
theoretical regret bounds of $O(\log T)$ for both versions, where $T$ is the
number of rounds, which is asymptotically tight. Our analysis uses the special
feature of the community exploration problem, which leads to improved
coefficients in the regret bounds compared with a simple application of some
existing results on combinatorial MABs. Moreover, we also discuss the
possibility of using the feedback in previous round to turn the problem into the
full information feedback model, which allows us to provide constant regret in
this case.

In summary, our contributions include: (a) proposing the study of the community
exploration problem to reflect the core of a number of real-world applications;
and (b) a systematic study of the problem with rigorous theoretical analysis
that covers offline non-adaptive, offline adaptive, online non-adaptive and
online adaptive cases, which model the real-world situations of adapting to
feedback and handling uncertainty.

\section{Problem Definition}
We model the problem with $m$ disjoint communities $C_1, \dots, C_m$ with
$C=\cup_{i=1}^m C_i$, where each community $C_i$ has $d_i$ members. Each
exploration (or visit) of one community $C_i$ returns a member of $C_i$
uniformly at random, and we have a total budget of $K$ for explorations. Since
we can trivially explore each community once when $K \le m$, we assume that $K
> m$.

We consider both the offline setting where the sizes of the
communities $d_1, \ldots, d_m$ are known, and the online setting where the sizes
of the communities are unknown.
For the offline setting, we further consider two different problems: (1)
non-adaptive exploration and (2) adaptive exploration. For the non-adaptive
exploration, the explorer needs to predetermine the budget allocation $\bm k$
before the exploration starts, while for the adaptive exploration, she can
sequentially select the next community to explore based on previous observations
(the members met in the previous community visits).
Formally, we use pair $(i, \tau)$ to represent the $\tau$-th exploration of
community $C_{i}$, called an {\em item}. Let $\cE = [m]\times[K]$ be the set of
all possible items.
A {\em realization} is a function $\phi\colon \cE \rightarrow C$ mapping every
possible item $(i, \tau)$ to a member in the corresponding community $C_i$,
and $\phi(i, \tau)$ represents the member met in the exploration $(i, \tau)$.
We use $\Phi$ to denote a random realization, and the randomness comes from the
exploration results.
From the description above, $\Phi$ follows the distribution such that $\Phi(i,\tau) \in C_i$ is selected uniformly
	at random from $C_i$ and is independent of all other $\Phi(i',\tau')$'s.

For a budget allocation $\bm k=(k_1, \dots, k_m)$ and a realization $\phi$, we
define the reward $R$ as the number of distinct members met, i.e., $R(\bm{k},
\phi) =\sum_{i=1}^{m}|\cup_{\tau=1}^{k_i}\{\phi(i,\tau)\}|$, where $|\cdot|$
is the cardinality of the set.
The goal of the {\em non-adaptive exploration} is to find an optimal budget
allocation $\bm{k}^{*} =(k^{*}_1,\ldots, k^{*}_m)$ with given budget $K$, which
maximizes the expected reward taken over all possible realizations, i.e.,
\begin{equation}
\label{eq:offline_budget_allocation_problem}
\bm{k}^{*} \in \argmax_{\bm{k}\colon \norm{\bm k}_1 \leq K} \mathbb{E}_{\Phi}\left[ R(\bm{k}, \Phi) \right].
\end{equation}

For the adaptive exploration, the explorer sequentially picks a community to
explore, meets a random member of the chosen community, then picks the next
community, meets another random member of that community, and so on, until the
budget is used up. After each selection, the observations so far can be
represented as a {\em partial realization} $\psi$, a function from the subset of
$\cE$ to $C = \cup_{i=1}^m C_i$. Suppose that each community $C_i$ has been
explored $k_i$ times. Then the partial realization $\psi$ is a function mapping
items in $\cup_{i=1}^m \{(i,1),\ldots, (i, k_i)\}$ (which is also called the
domain of $\psi$, denoted as $\dom(\psi)$) to members in communities.
The partial realization $\psi$ records the observation on the sequence of
explored communities and the members met in this sequence.
We say that a partial realization $\psi$ is consistent with realization $\phi$, denoted as $\phi \sim \psi$, 
	if for all item $(i,\tau)$ in the domain of $\psi$, we have $\psi(i, \tau) = \phi(i,\tau)$.
The  strategy to
explore the communities adaptively is encoded as a policy.
The policy, denoted as $\pi$, is a
function mapping $\psi$ to an item in $\cE$, specifying which community to explore next
under the partial realization. 
Define $\pi_K(\phi) =
(k_1,\ldots, k_m)$, where $k_i$ is the times the community $C_i$ is explored via
policy $\pi$ under realization $\phi$ with budget $K$. 
More specifically, starting from the partial realization $\psi_0$ with empty
domain, for every current partial realization $\psi_s$ at step $s$, policy $\pi$
determines the next community $\pi(\psi_s)$ to explore, meet the member
$\phi(\pi(\psi_s))$, such that the new partial realization $\psi_{s+1}$ is
adding the mapping from $\pi(\psi_s)$ to $\phi(\pi(\psi_s))$ on top of $\psi_s$.
This iteration continues until the communities have been explored $K$ times, and
$\pi_K(\phi) = (k_1,\ldots, k_m)$ denotes the resulting exploration vector. The
goal of the adaptive exploration is to find an optimal policy $\pi^*$ to
maximize the expected adaptive reward, i.e.,
\begin{equation}
  \label{eq:adaptive_allocation_problem}
 \pi^* \in \argmax_{\pi} \mathbb{E}_{\Phi}\left[R(\pi_K(\Phi), \Phi) \right]. 
\end{equation}
We next consider the online setting of community exploration. The learning
process proceeds in discrete rounds. Initially, the size of communities $\bm{d}
= (d_1,\ldots, d_m)$ is unknown.
In each round $t\ge 1$, the learner needs to determine an allocation or a policy 
(called an {\em ``action''}) based on the previous-round observations to explore
communities (non-adaptively or adaptively). 
When an action is played,
the sets of encountered members for every community are observed as the {\em
  feedback} to the player. 
A learning algorithm $A$ aims to cumulate as much reward (i.e., number of
distinct members) as possible by selecting actions properly at each round. The
performance of a learning algorithm is measured by the {\em cumulative regret}. Let
$\Phi_t$ be the realization at round $t$.
If we explore the communities with predetermined budget allocation in each
round, the $T$-round (non-adaptive) regret of a learning algorithm $A$ is defined as
\begin{equation}
  \label{eq:non_adaptive_regret_definition}
  \small
 \text{Reg}^{A}_{\bm{\mu}}(T) = \mathbb{E}_{\Phi_1,\ldots, \Phi_T}\left[\sum_{t=1}^{T}R(\bm{k}^{*}, \Phi_t) - R(\bm{k}^{A}_t, \Phi_t) \right],
\end{equation}
where the budget allocation $\bm{k}^{A}_t$ is selected by algorithm $A$ in round
$t$. 
If we explore the
communities adaptively in each round, then the $T$-round (adaptive) regret of a learning
algorithm $A$ is defined as
\begin{equation}
  \label{eq:adaptive_regret_definition}
  \small
 \text{Reg}^{A}_{\bm{\mu}}(T) = \mathbb{E}_{\Phi_1,\ldots, \Phi_T}\left[\sum_{t=1}^{T}R(\pi^{*}_K(\Phi_t), \Phi_t) - R(\pi^{A,t}_K(\Phi_t), \Phi_t) \right],
\end{equation}
where $\pi^{A,t}$ is a policy selected by algorithm $A$ in round $t$.
The goal of the learning problem is to design a learning algorithm $A$ which minimizes the regret defined in~\eqref{eq:non_adaptive_regret_definition} and~\eqref{eq:adaptive_regret_definition}.

\section{Offline Optimization for Community Exploration}\label{sec:offline}
\subsection{Non-adaptive Exploration Algorithms}
If $C_i$ is explored $k_i$ times, each member in $C_i$ is encountered at least
once with probability $1 - (1 - 1/d_i)^{k_i}$ . Thus we have
$\mathbb{E}_{\Phi}[\abs{\{\Phi(i,1), \ldots, \Phi(i, {k_i})\} }] = d_i (1 - (1 -
1/d_i)^{k_i})$. Hence $\mathbb{E}_{\Phi}\left[ R(\bm{k}, \Phi) \right]$ is a
function of only the budget allocation $\bm{k}$ and the size $\bm{d} =
(d_1,\ldots, d_m)$ of all communities. Let $\mu_i = 1 / d_i$, and vector
$\bm{\mu} = (1/d_1, \ldots, 1/d_m)$. Henceforth, we treat $\bm{\mu}$ as the
parameter of the problem instance, since it is bounded with $\bm{\mu}\in
[0,1]^m$. Let $r_{\bm{k}}(\bm{\mu}) = \mathbb{E}_{\Phi}[R(\bm{k}, \Phi)]$ be the
expected reward for the budget allocation $\bm{k}$. Based on the above
discussion, we have
\begin{equation} \label{eq:expectedreward}
  r_{\bm{k}}(\bm{\mu}) = \sum_{i=1}^m d_i (1 - (1 - 1/d_i)^{k_i}) = \sum_{i=1}^m (1 - (1 - \mu_i)^{k_i}) / \mu_i.
\end{equation}
Since $k_i$ must be integers, a traditional method like
{\em Lagrange Multipliers} cannot be applied to solve the optimization problem
defined in Eq.~\eqref{eq:offline_budget_allocation_problem}. We propose a {\em
  greedy method} consisting of $K$ steps to compute the feasible $\bm{k}^*$.
The greedy method is described in
Line~\ref{line:budget_allocation_start}-\ref{line:budget_allocation_end} of Algo.~\ref{algo:non_adaptive_exploration}.

\begin{algorithm}[t]
  \caption{Non-Adaptive community exploration with optimal budget allocation}\label{algo:non_adaptive_exploration}
  \begin{algorithmic}[1]
\Procedure{\texttt{CommunityExplore}}{\{$\mu_1,\dots, \mu_m$\}, $K$, non-adaptive}
\State{For $i\in [m]$, $k_i\leftarrow 0$}\label{line:budget_allocation_start}\Comment{Line~\ref{line:budget_allocation_start}-\ref{line:budget_allocation_end}:
budget allocation}
\For{$s = 1, \dots, K$}
\State{$i^*\leftarrow$ a random elements in $\argmax_{i}(1 - \mu_i)^{k_i}$}\Comment{$O(\log m)$ via using priority queue}
\State{$k_{i^*}\leftarrow k_{i^*} + 1$}\label{line:budget_allocation_end}
\EndFor
\State{For $i\in [m]$, explore $C_i$ for $k_i$ times, and put the uniformly met
  members in multi-set $\cS_i$}
\EndProcedure
  \end{algorithmic}
\end{algorithm}

\begin{restatable}{theorem}{optimalitya}\label{thm:nonadaptive_greedy_is_optimal}
 The greedy method obtains an optimal budget allocation. 
\end{restatable}
The time complexity of the greedy method is $O(K\log m)$, which is not efficient
for large $K$. We find that starting from the initial allocation $ k_i =
\ceil*{\frac{(K - m) / \ln (1 - \mu_i)}{\sum_{j=1}^m 1 / \ln (1 - \mu_j)}}$, the
greedy method can find the optimal budget allocation in $O(m\log
m)$\footnote{We thank Jing Yu from School of Mathematical Sciences at Fudan
  University for her method to find a good initial allocation, which leads to a
  faster greedy method.}.
\OnlyInFull{(See
Appendix~\ref{app:improved_budget_allocation})}\OnlyInShort{(See supplementary
materials.)}

\subsection{Adaptive Exploration Algorithms}
With a slight abuse of notations, we also define $r_{\pi}(\bm{\mu}) =
\mathbb{E}_{\Phi}\left[ R(\pi_K(\Phi), \Phi) \right]$, since the expected reward
is the function of the policy $\pi$ and the vector $\bm{\mu}$.
Define $c_i(\psi)$ as the number of distinct members we met in community $C_i$
under partial realization $\psi$. Then $1 - c_i(\psi)/d_i$ is the
probability that we can meet a new member in the community $C_i$ if we explore
community $C_i$ one more time. A natural approach 
is to explore community $C_{i^*}$
such that $i^*\in \argmax_{i\in[m]} 1 - c_i(\psi)/d_i$ when we have partial
realization $\psi$. We call such policy as the greedy policy $\pi^g$. 
The adaptive community exploration with greedy policy is described in Algo.~\ref{algo:adaptive_exploration}.
One could show that our reward
function is actually an {\em adaptive submodular} function, for which the greedy
policy is guaranteed to achieve at least $(1 \!- \!1/e)$ of the maximized expected
reward~\cite{golovin2011adaptive}.
However, the following theorem shows that for our community exploration 
problem, our greedy policy is in fact {\em optimal}.
\begin{algorithm}[t]
  \caption{Adaptive community exploration with greedy policy}\label{algo:adaptive_exploration}
  \begin{algorithmic}[1]
\Procedure{\texttt{CommunityExplore}}{\{$\mu_1,\dots, \mu_m$\},
  $K$, adaptive}
\State{For $i\in [m]$, $\cS_i\leftarrow \emptyset$, $c_i\leftarrow 0$}\label{line:adaptive_explore_start}\Comment{Line~\ref{line:adaptive_explore_start}-\ref{line:adaptive_explore_end}: adaptively
  explore communities with policy $\pi^g$}
\For{$s=1,\dots, K$}
\State{$i^*\leftarrow$ a random elements in $\argmax_{i}1 - \mu_ic_i$}
\State{$v\leftarrow$ a random member met when $C_{i^*}$ is explored}
\State{\textbf{if} $v\notin \cS_{i^*}$ \textbf{then} $c_{i^*}\leftarrow c_{i^*} + 1$}\Comment{$v$ is not met before} 
\State{$\cS_{i^*}\leftarrow \cS_{i^*}\cup \{v\}$}\label{line:adaptive_explore_end}
\EndFor
\EndProcedure
  \end{algorithmic}
\end{algorithm}

\begin{restatable}{theorem}{optimalityb}\label{thm:greedy_policy_is_optimal}
	Greedy policy is the optimal policy for our adaptive exploration problem.  
\end{restatable}
\textbf{Proof sketch.}
Note that the greedy policy chooses the next community only based on the
fraction of unseen members. It does not care which members are already met.
Thus, we define $s_i$ as the percentage of members we have not met in a
community $C_i$. We introduce the concept of {\em status}, denoted as $\bm{s} =
\left(s_1,\dots, s_m\right)$. The greedy policy chooses next community based on
the current status. In the proof, we further extend the definition of reward
with a non-decreasing function $f$ as $ R(\bm{k}, \phi)
=f\left(\sum_{i=1}^{m}\abs{\bigcup_{\tau=1}^{k_i}\{\phi(i,\tau)\}}\right)$.
Note that the reward function corresponding to the original community
exploration problem is simply the identity function $f(x)=x$.
Let $F_{\pi}(\psi, t)$ denote the expected {\em marginal gain} when we further
explore communities for $t$ steps with policy $\pi$ starting from a partial
realization $\psi$. We want to prove that for all $\psi$, $t$ and $\pi$,
$F_{\pi^g} (\psi,t)\ge F_\pi(\psi,t)$, where $\pi^g$ is the greedy policy and
$\pi$ is an arbitrary policy. If so, we simply take $\psi = \emptyset$, and
$F_{\pi^g} (\emptyset,t)\ge F_\pi(\emptyset,t)$ for every $\pi$ and $t$ exactly
shows that $\pi^g$ is optimal. We prove the above result by an induction on $t$.

Let $C_i$ be the community chosen by $\pi$ based on the partial realization $\psi$.
Define $c(\psi) = \sum_i c_i(\psi)$ and $\Delta_{\psi, f} = f(c(\psi) + 1) -
f(c(\psi))$. We first claim that $F_{\pi^g} (\psi,1)\ge F_\pi(\psi,1)$ holds for
all $\psi$ and $\pi$ with the fact that $F_{\pi}(\psi, 1) = (1 -
\mu_ic_i(\psi))\Delta_{\psi, f}$. Note that the greedy policy $\pi^g$ chooses
$C_{i^*}$ with $i^*\in \argmax_{i} (1 - \mu_i c_i(\psi))$. Hence, $F_{\pi^g}
(\psi,1)\ge F_\pi(\psi,1)$.

Next we prove that $F_{\pi^g} (\psi,t+1)\ge F_\pi(\psi,t+1)$ based on the
assumption that $F_{\pi^g} (\psi,t')\ge F_\pi(\psi,t')$ holds for all $\psi$,
$\pi$, and $t'\le t$. An important observation is that $F_{\pi^g}(\psi, t)$ has
equal value for any partial realization $\psi$ associated with the same status
$\bm{s}$ since the status is enough for the greedy policy to determine the
choice of next community. Formally, we define $F_{g}(\bm{s}, t) =
F_{\pi^g}(\psi, t)$ for any partial realization that satisfies $\bm{s} = (1 -
c_1(\psi)/d_1,\dots, 1 - c_m(\psi)/d_m)$. Let $C_{i^*}$ denote the community
chosen by policy $\pi^g$ under realization $\psi$, i.e., $i^*\in \argmax_{i\in
[m]} 1 - \mu_ic_i(\psi)$. Let $\bI_i$ be the $m$-dimensional unit vector with
one in the $i$-th entry and zeros in all other entries. We show that
 \begin{align*}
  F_{\pi}(\psi, t + 1) &\leq c_i(\psi) \cdot \mu_i F_{g}(\bm{s}, t) + (d_i - c_i(\psi)) \cdot \mu_i F_{g}(\bm{s} - \mu_i \bI_i, t) + (1 - \mu_ic_i(\psi))\Delta_{\psi, f} \\
   & \leq \mu_{i^*}c_{i^*}(\psi)F_{g}(\bm{s}, t) + (1 - \mu_{i^*}c_{i^*}(\psi))F_{g}(\bm{s} - \mu_{i^*} \bI_{i^*}, t)  + (1 - \mu_{i^*}c_{i^*}(\psi))\Delta_{\psi, f} \\
   & = F_{g}(\bm{s}, t + 1) = F_{\pi^g}(\psi, t + 1). 
 \end{align*}
The first line is derived directly from the definition and the assumption. The
key is to prove the correctness of Line 2 in above inequality. It indicates that
if we choose a sub-optimal community at first, and then we switch back to the
greedy policy, the expected reward would be smaller. The proof is nontrivial and
relies on a careful analysis based on the stochastic transitions among status
vectors. \OnlyInShort{We leave detailed analysis in the supplementary materials.}
Note that the reward function $r_{\pi}(\bm{\mu})$ is not necessary adaptive
submodular if we extend the reward with the non-decreasing function $f$. Hence,
a $(1 - 1/e)$ guarantee for adaptive submodular
function~\cite{golovin2011adaptive} is not applicable in this scenario. Our
analysis scheme can be applied to any adaptive problems with similar structures.

\section{Online Learning for Community Exploration}
The key of the learning algorithm is to estimate the community sizes. The size
estimation problem is defined as inferring unknown set size $d_i$ from random
samples obtained via uniformly sampling {\em with replacement} from the set
$C_i$. Various estimators have been proposed~\cite{finkelstein1998423,
bressan2015simple, christman1994sequential, katzir2011estimating} for the
estimation of $d_i$. The core idea of estimators in~\cite{bressan2015simple,
katzir2011estimating} are based on ``{\em collision counting}''.
Let $(u, v)$ be an {\em unordered pair} of two random elements from $C_i$ and
$Y_{u, v}$ be a {\em pair collision} random variable that takes value 1 if $u =
v$ (i.e., $(u,v)$ is {\em a collision}) and $0$ otherwise. It is easy to verify
that $\mathbb{E}[Y_{u,v}] = 1/ d_i = \mu_i$. Suppose we {\em independently} take
$T_i$ pairs of elements from $C_i$ and $X_i$ of them are collisions. Then
$\mathbb{E}[X_i/T_i] = 1 / d_i = \mu_i$. The size $d_i$ can be estimated by $T_i
/ X_i$ (the estimator is valid when $X_i > 0$).

\begin{algorithm}[t]
  \caption{Combinatorial Lower Confidence Bound (CLCB) algorithm}\label{algo:CLCB_algorithm}
  \begin{algorithmic}[1]
   \Require{budget for each round $K$, $\mathsf{method}$ (non-adaptive or adaptive)} 
   \State{For $i \in [m]$, $T_i\leftarrow 0 $ (number of pairs), $X_i\leftarrow
     0$ (collision counting), $\hat{\mu}_i \leftarrow 0 $ (empirical mean)}
   \For{$t = 1, 2, 3, \dots
     $}\label{line:online_learning_start}\Comment{Line~\ref{line:online_learning_start}-\ref{line:online_learning_end}:
   online learning}
   \State{For $i\in [m]$, $\rho_i\leftarrow \sqrt{\frac{3\ln t}{2T_i}}$ ($\rho_i
     = 0$ if $T_i = 0$)}\label{line:radius} \Comment{confidence radius}
   \State{For $i\in [m]$, $\ubar{\mu}_i\leftarrow\max\{0, \hat{\mu}_i - \rho_i\}$}\Comment{lower confidence bound}
   \State{$\{\cS_1,\dots, \cS_m\}\leftarrow \textproc{\texttt{CommunityExplore}}(\{\ubar{\mu}_1, \dots,
     \ubar{\mu}_m\}, K, \mathsf{method})$}\label{line:feedback}
   \Comment{$\cS_i$: set of met members}
   \State{For $i\in [m]$, $T_i\leftarrow T_i + \floor{\abs{\cS_i}/2}$ }\label{line:update_pairs}
   \Comment{update number of (member) pairs we observe}
   \State{For $i\in [m]$, $X_i\leftarrow X_i +
     \sum_{x=1}^{\floor{\abs{\cS_i}} / 2}\bOne\{\cS_i[{2x - 1}] = 
     \cS_i[{2x}]\}$}\label{line:count_collision}\Comment{$\cS_i[x]$: $x$-th element in
     $\cS_i$}
   \State{For $i\in [m]$ and $\abs{\cS_i} > 1$, $\hat{\mu}_i\leftarrow X_i / T_i$}\label{line:online_learning_end}
   \Comment{update empirical mean}
   \EndFor
  \end{algorithmic}
\end{algorithm}

We present our CLCB algorithm in Algorithm~\ref{algo:CLCB_algorithm}. In the
algorithm, we maintain an unbiased estimation of $\mu_i$ instead of $d_i$ for
each community $C_i$ for the following reasons. Firstly, $T_i/X_i$ is not an
unbiased estimator of $d_i$ since $\mathbb{E}[T_i/X_i] \geq d_i$ according to
the Jensen's inequality. Secondly, the upper confidence bound of $T_i/X_i$
depends on $d_i$, which is unknown in our online learning problem. Thirdly, we
need at least $(1 + \sqrt{8d_i\ln1/\delta + 1})/2$ uniformly sampled elements in
$C_i$ to make sure that $X_i > 0$ with probability at least $1 - \delta$. We
feed the lower confidence bound $\ubar{\mu}_i$ to the exploration process since
our reward function increases as $\mu_i$ decreases. The idea is similar to CUCB
algorithm~\cite{CWYW16}. The lower confidence bound is small if community $C_i$ is not
explored often ($T_i$ is small). Small $\ubar{\mu}_i$ motivates us to explore
$C_i$ more times. The {\em feedbacks} after the exploration process at each round
are the sets of encountered members $\cS_1,\dots, \cS_m$ in communities
$C_1,\dots, C_m$ respectively. 
Note that for each $i\in[m]$, all pairs of elements in $\cS_i$, namely $\{(x,
y)\mid x \leq y, x\in \cS_i, y\in \cS_i\backslash\{x\}\}$ are not mutually
independent. Thus, we only use $\floor{\abs{\cS_i}/2}$ independent pairs.
Therefore, $T_i$ is updated as $T_i +
\floor{\abs{\cS_i}/2}$ at each round. 
In each round, the community exploration could either be non-adaptive or
adaptive, and the following regret analysis separately discuss these two cases.

\subsection{Regret Analysis for the Non-adaptive Version}

The non-adaptive bandit learning model fits into the general combinatorial
multi-armed bandit (CMAB) framework of~\cite{CWYW16,
wang2017improving} that deals with nonlinear reward functions. In particular, we
can treat the pair collision variable in each community $C_i$ as a base arm, and
our expected reward in Eq.~\eqref{eq:expectedreward} is non-linear, and it
satisfies the monotonicity and bounded smoothness properties (See
Properties~\ref{pro:monotone} and \ref{pro:bounded_smoothness}). However,
directly applying the regret result from \cite{CWYW16,
wang2017improving} will give us an inferior regret bound for two reasons.
First, in our setting, in each round we could have multiple sample feedback for
each community, meaning that each base arm could be observed multiple times,
which is not directly covered by CMAB. Second, to use the regret result
in~\cite{CWYW16, wang2017improving}, the bounded smoothness
property needs to have a bounded smoothness constant independent of the actions,
but we can have a better result by using a tighter form of bounded smoothness
with action-related coefficients. Therefore, in this section, we provide a
better regret result by adapting the regret analysis in
\cite{wang2017improving}.

We define the gap $\Delta_{\bm{k}} = r_{\bm{k}^*}(\bm{\mu}) -
r_{\bm{k}}(\bm{\mu})$ for all action $\bm{k}$ satisfying $\sum_{i=1}^m k_i = K$.
For each community $C_i$, we define $\Delta^{i}_{\min} = \min_{\Delta_{\bm{k}} >
0, k_i > 1} \Delta_{\bm{k}}$ and $\Delta^{i}_{\max} = \max_{\Delta_{\bm{k}} > 0,
k_i > 1} \Delta_{\bm{k}}$. As a convention, if there is no action $\bm{k}$ with
$k_i > 1$ such that $\Delta_{\bm{k}} > 0$, we define $\Delta^{i}_{\min} =
\infty$ and $\Delta^{i}_{\max} = 0$. Furthermore, define $\Delta_{\min} =
\min_{i\in [m]} \Delta^{i}_{\min}$ and $\Delta_{\max} = \max_{i\in
[m]}\Delta^{i}_{\max}$. Let $K^{\prime} = K - m + 1$. We have the regret for
Algo.~\ref{algo:CLCB_algorithm} as follows.

\begin{restatable}{theorem}{regretbounda}\label{thm:regret_bound_non_adaptive_exploration_1}
  Algo.~\ref{algo:CLCB_algorithm} with non-adaptive exploration method has regret as follows.
 \begin{align}\label{eq:non_adaptive_regret_bound_a}
    \text{Reg}_{\bm{\mu}}(T) &\leq \sum_{i=1}^{m}\frac{48{K'\choose 2}K\ln T}{\Delta^{i}_{\min}}  + 2{K^{\prime}\choose 2}m + \frac{\floor*{\frac{K^{\prime}}{2}}\pi^2}{3}m\Delta_{\max} 
    = O\left( \sum_{i=1}^{m} \frac{K'^3\log T}{\Delta^{i}_{\min}} \right).
\end{align}
\end{restatable}

The proof of the above theorem is an adaption of the proof of Theorem 4
in~\cite{wang2017improving}, and the full proof details as well as the detailed
comparison with the original CMAB framework result are included in the
supplementary materials. We briefly explain our adaption that leads to the
regret improvement. We rely on the following monotonicity and 1-norm bounded
smoothness properties of our expected reward function $r_{\bm{k}}(\bm{\mu})$,
similar to the ones in \cite{CWYW16, wang2017improving}.
\begin{property}[Monotonicity] \label{pro:monotone}
	The reward function $r_{\bm{k}}(\bm{\mu})$ is monotonically decreasing, i.e.,
	for any two vectors $\bm{\mu} = (\mu_1,\dots, \mu_m)$ and $\bm{\mu}^{\prime} =
	(\mu^{\prime}_1, \dots,\mu^{\prime}_m)$, we have $r_{\bm{k}}(\bm{\mu})\geq
	r_{\bm{k}}(\bm{\mu}^{\prime})$ if $\mu_i\leq \mu^{\prime}_i \ \forall i\in [m]$.
\end{property}

\begin{property}[1-Norm Bounded Smoothness]\label{pro:bounded_smoothness}
	The reward function $r_{\bm{k}}(\bm{\mu})$ satisfies the 1-norm bounded
	smoothness property, i.e., for any two vectors $\bm{\mu} = (\mu_1,\cdots,
	\mu_{m})$ and $\bm{\mu}^{\prime} =(\mu^{\prime}_1, \cdots, \mu^{\prime}_{m})$,
	we have $|r_{\bm{k}}(\bm{\mu}) - r_{\bm{k}}(\bm{\mu}^{\prime})| \leq \sum_{i =
		1}^{m}{k_i\choose 2}|\mu_i - \mu^{\prime}_i|\leq {K' \choose 2}\sum_{i =
		1}^{m}|\mu_i - \mu^{\prime}_i|$.
\end{property}

We remark that if we directly apply the CMAB regret bound of Theorem 4
in~\cite{wang2017improving}, we need to revise the update procedure in
Lines~\ref{line:update_pairs}-\ref{line:online_learning_end} of
Algo.~\ref{algo:CLCB_algorithm} so that each round we only update one
observation for each community $C_i$ if $|\cS_i| > 1$. Then we would obtain a
regret bound $O\left( \sum_i \frac{K'^4 m\log T}{\Delta^{i}_{\min}} \right)$,
which means that our regret bound in Eq.~\eqref{eq:non_adaptive_regret_bound_a}
has an improvement of $O(K'm)$. This improvement is exactly due to the reason we
give earlier, as we now explain with more details.

For all the random variables introduced in Algo.~\ref{algo:CLCB_algorithm}, we
add the subscript $t$ to denote their value at the {\em end} of round $t$. For
example, $T_{i,t}$ is the value of $T_i$ at the end of round $t$. First, the
improvement of the factor $m$ comes from the use of a tighter bounded
smoothness in Property~\ref{pro:bounded_smoothness}, namely, we use the
bound $\sum_{i = 1}^{m}{k_i\choose 2}|\mu_i - \mu^{\prime}_i|$ instead of ${K'
\choose 2}\sum_{i = 1}^{m}|\mu_i - \mu^{\prime}_i|$.
The CMAB framework in~\cite{wang2017improving} requires the bounded
smoothness constant to be independent of actions. So to apply Theorem 4
in~\cite{wang2017improving}, we have to use the bound ${K' \choose 2}\sum_{i =
1}^{m}|\mu_i - \mu^{\prime}_i|$. However, in our case, when using bound $\sum_{i
= 1}^{m}{k_i\choose 2}|\mu_i - \mu^{\prime}_i|$, we are able to utilize the
fact $\sum_{i=1}^{m}{k_{i}\choose 2} \leq {K^{\prime}\choose 2}$ to improve the
result by a factor of $m$.
Second, the improvement of the $O(K')$ factor, more precisely a factor of
$(K'-1)/2$, is achieved by utilizing multiple feedback in a single round and a
more careful analysis of the regret utilizing the property of the right Riemann
summation. 
Specifically, let $\Delta_{\bm{k}_t} = r_{\bm{k}^*}({\bm{\mu}}) -
r_{\bm{k}_t}({\bm{\mu}})$ be the reward gap. When the estimate is within the
confidence radius, we have $\Delta_{\bm{k}_t} \leq \sum_{i=1}^{m}\frac{c(k_{i,t}
- 1)}{2}/\sqrt{T_{i,t-1}}\leq c\sum_{i=1}^{m}\floor{k_{i,t} /
2}/\sqrt{T_{i,t-1}}$, where $c$ is a constant. In
Algo.~\ref{algo:CLCB_algorithm}, we have $T_{i,t} = T_{i,t-1} +
\floor{k_{i,t}/2}$ because we allow multiple feedback in a single round. Then
$\sum_{t\geq 1, T_{i,t}\leq L_i(T)}\floor{k_{i,t} / 2}/\sqrt{T_{i,t-1}}$ is the
form of a right Riemann summation, which achieves the maximum value when
$k_{i,t} = K^{\prime}$. Here $L_i(T)$ is a $\ln T$ function with some constants
related with community $C_i$. Hence the regret bound
$\sum_{t=1}^{T}\Delta_{\bm{k}_t} \leq c\sum_{i=1}^{m}\sum_{t\geq 1, T_{i,t}\leq
L_i(T)} \floor{\frac{k_{i,t}}{2}} /\sqrt{T_{i,t-1}} \leq
2c\sum_{i=1}^{m}\sqrt{L_i(T) }$. However, if we use the original CMAB framework,
we need to set $T_{i,t} = T_{i,t-1} + \bOne\{k_{i,t} > 1\}$. In this case, we
can only bound the regret as $\sum_{t=1}^{T}\Delta_{\bm{k}_t} =
c\sum_{i=1}^m\sum_{t\geq 1, T_{i,t}\leq L_i(T)}(k_{i,t} -1)/ 2\sqrt{T_{i,t-1}}
\leq 2c {\frac{K^{\prime} - 1}{2}}\sum_{i=1}^{m}\sqrt{L_i(T)}$, leading to an
extra factor of ${(K^{\prime}-1)/2}$.

\textbf{Justification for Algo.~\ref{algo:CLCB_algorithm}.} In
Algo.~\ref{algo:CLCB_algorithm}, we only use the members in current round to
update the estimator. This is practical for the situation where the member
identifiers are changing in different rounds for privacy protection. Privacy
gains much attention these days.
Consider the online advertising scenario we explain in
the introduction. Whenever a user clicks an advertisement, the advertiser would
store the user information (e.g. Facebook ID, IP address etc.) to identify the
user and correlated with past visits of the user. If such user identifiers are
fixed and do not change, the advertiser could easily track user behavior, which
may result in privacy leak. A reasonable protection for users is to periodically
change user IDs (e.g. Facebook can periodically change user hash IDs, or users
adopt dynamic IP addresses, etc.), so that it is difficult for the advertiser to
track the same user over a long period of time. Under such situation, it may be
likely that our learning algorithm can still detect ID collisions within the
short period of each learning round, but cross different rounds, collisions may
not be detectable due to ID changes.

\textbf{Full information feedback.} Now we consider the scenario where the member identifiers are fixed 
	over all rounds, and design an algorithm with a constant regret
	bound. 
Our idea is to ensure that we can observe at least one pair of members in
every community $C_i$ in each round $t$. We call such guarantee as {\em full
information feedback}. 
If we only use members revealed in current round, we
cannot achieve this goal since we have no observation of new pairs for a community $C_i$
	when $k_{i} = 1$. 
To achieve full information feedback, we use at least one sample from the previous round 
	to form a pair with a sample in the current round to generate a valid pair collision observation.
In particular, we revise the
Line~\ref{line:radius}, \ref{line:update_pairs}, and \ref{line:count_collision} as follows. Here we use
$u_{0}$ to represent the last member in $\cS_i$ in the previous round (let
$u_{0}=\text{null}$ when $t=1$) and $u_x (x > 0)$ to represent the $x$-th
members in $\cS_i$ in the current round.
The revision of Line~\ref{line:radius} implies that we use the empirical mean
$\hat{\mu}_i = X_i / T_i$ instead of the lower confidence bound in the function
\texttt{\textproc{CommunityExplore}}.
\begin{equation}\label{eq:revision_2}
  \begin{split}
& \text{Line~\ref{line:radius}:~~ For } i\in [m], \rho_{i} = 0;~~~\text{Line~\ref{line:update_pairs}:~~ For } i\in [m], T_i\leftarrow T_i + \abs{\cS_i} - \bOne\{t = 1\}, \\ 
& \text{Line~\ref{line:count_collision}:~~ For } i\in [m], X_{i} \leftarrow X_i+ \sum\nolimits_{x=0}^{\abs{\cS_i} - 1}\bOne\{u_{x} = u_{x+1}\}.\\
  \end{split}
\end{equation}
\begin{theorem}\label{thm:non_adaptive_full_information}
  With the full information feedback revision in Eq.~\eqref{eq:revision_2},
  Algo.~\ref{algo:CLCB_algorithm} with non-adaptive exploration method has a constant regret
  bound. Specifically,
  \begin{equation*}
   \text{Reg}_{\bm{\mu}}(T) \leq  \left(2 + 2me^2K'^2(K'-1)^2 / \Delta^2_{\min}\right)\Delta_{\max}.
  \end{equation*}
\end{theorem}

Note that we cannot apply the Hoeffding bound in~\cite{hoeffding1963probability}
directly since the random variables $\bOne\{u_{x} = u_{x+1}\}$ we obtain during
the online learning process are not mutually independent. Instead, we apply a
concentration bound in~\cite{Dubhashi2009CMA} that is applicable to variables
that have local dependence relationship.

\subsection{Regret Analysis for the Adaptive Version}\label{sec:online_adaptive}

For the adaptive version, we feed the lower confidence bound 
$\ubar{\bm{\mu}}_{t}$ into the adaptive community exploration procedure, namely
	$\textproc{\texttt{CommunityExplore}}(\{\ubar{\mu}_1, \dots,
	\ubar{\mu}_m\}, K, \textrm{adaptive})$ in round $t$.
We denote the policy implemented by this procedure as $\pi^t$.
Note that both $\pi^g$ and $\pi^t$ are based on the greedy procedure 
	$\textproc{\texttt{CommunityExplore}}(\cdot , K, \textrm{adaptive})$.
	The difference is that $\pi^g$ uses the true parameter $\bm{\mu}$ while
	$\pi^t$ uses the lower bound parameter $\ubar{\bm{\mu}}_t$. 
More specifically,  given a
partial realization $\psi$, the community chosen by $\pi^t$ is $C_{i^*}$ where
$i^*\in \argmax_{i\in [m]} 1 - c_i(\psi)\ubar{\mu}_{i,t}$. Recall that
$c_i(\psi)$ is the number of distinct encountered members  in community $C_i$ under partial
realization $\psi$. 

\OnlyInFull{
Similar to $\pi^g$, the policy $\pi^t$
also chooses next community to explore based on current {\em status}. Let
$\bm{s} = (s_1, \dots, s_m) = (1 - c_1(\psi)\mu_1, \dots, 1 - c_m(\psi)\mu_m)$
be the corresponding status to the partial realization $\psi$. Here $s_i$ is the
percentage of unmet members in the community $C_i$. For any partial realization
$\psi$ having status $\bm{s}$, the policy $\pi^t$ choose $C_{i^*}$ to explore,
where $i^*\in \argmax_{i\in [m]} (\ubar{\mu}_{i,t}/\mu_i)s_i + (\mu_i -
\ubar{\mu}_i)/\mu_i$. When $\ubar{\mu}_{i,t}\leq \mu_i$, we have
$(\ubar{\mu}_{i,t}/\mu_i)s_i + (\mu_i - \ubar{\mu}_i)/\mu_i \geq s_i$, which
means that the percentage of unmet members in $C_i$ is overestimated by $\pi^t$.
}

We first properly define the metrics $\Delta^{i, k}_{\min}$ and
$\Delta^{(k)}_{\max}$ used in the regret bound as follows. Consider a specific
full realization $\phi$ where $\{\phi(i, 1), \dots, \phi(i, d_i)\}$ are $d_i$
distinct members in $C_i$ for $i\in [m]$. The realization $\phi$ indicates that
we will obtain a new member in the first $d_i$ exploration of community $C_i$.
Let $U_{i,k}$ denote the number of times community $C_i$ is selected by policy
$\pi^g$ in the first $k-1 (k > m)$ steps under the special full realization $\phi$
we define previously. We define $\Delta^{i, k}_{\min} = (\mu_iU_{i,k} -
\min_{j\in [m]}\mu_jU_{j,k})/U_{i,k}$. 
Conceptually, the value $\mu_iU_{i,k} - \min_{j\in
[m]}\mu_jU_{j,k}$ is gap in the expected reward of the next step between selecting
	a community by $\pi^g$ (the optimal policy) and selecting community $C_i$, 
	when we already meet $U_{j,k}$ distinct members in $C_j$ for $j\in [m]$.
When $\mu_iU_{i,k} = \min_{j\in [m]}\mu_jU_{j,k}$, we define
$\Delta^{i,k}_{\min} = \infty$.
Let $\pi$ be another policy that chooses the same sequence of communities as
$\pi^g$ when the number of met members in $C_i$ is no more than $U_{i,k}$ for
all $i\in [m]$. Note that policy $\pi$ chooses the same communities as $\pi^g$
in the first $k-1$ steps under the special full realization $\phi$.
Actually, the policy $\pi$ is the same as $\pi^g$ for at least $k-1$ steps.
We use $\Pi_k$ to denote the set of all such policies. We define
$\Delta^{(k)}_{\max}$ as the maximum reward gap between the policy $\pi\in
\Pi_k$ and the optimal policy $\pi^g$, i.e., $\Delta^{(k)}_{\max} = \max_{\pi\in
\Pi_k} r_{\pi^g}(\bm{\mu}) - r_{\pi}(\bm{\mu})$. Let $D = \sum_{i=1}^{m}d_i$.

\begin{restatable}{theorem}{regretboundc}\label{thm:regret_bound_adaptive_exploration}
  Algo.~\ref{algo:CLCB_algorithm} with adaptive exploration method has regret as follows.
\begin{align} \label{eq:adaptive_regret_bound_a}
 &\text{Reg}_{\bm{\mu}}(T) \leq \left(  \sum_{i=1}^{m}\sum_{k = m + 1}^{\min\{K, D\}} \frac{6\Delta^{(k)}_{\max}}{(\Delta^{i,k}_{\min})^2}\right)\ln T + 
 \frac{\floor{\frac{K^{\prime}}{2}}\pi^2}{3}\sum_{i=1}^{m}\sum_{k = m + 1}^{\min\{K, D\}}\Delta^{(k)}_{\max}.
\end{align}
\end{restatable}

\begin{theorem}\label{thm:adaptive_full_information}
  With the full information feedback revision in Eq.~\eqref{eq:revision_2},
  Algo.~\ref{algo:CLCB_algorithm} with adaptive exploration method has a constant regret
  bound. Specifically,
  \begin{equation*}
   \text{Reg}_{\bm{\mu}}(T) \leq  \sum\nolimits_{i=1}^{m}\sum\nolimits_{k = m + 1}^{\min\{K, D\}} \left(2/\varepsilon^4_{i, k} + 1\right)\Delta^{(k)}_{\max}.
  \end{equation*}
  where $\varepsilon_{i, k}$ is defined as (here $i^*_k\in \argmin_{i\in [m]} \mu_iU_{i,k}$)
  \begin{equation*}
    \varepsilon_{i, k} \triangleq (\mu_iU_{i,k} - \mu_{i^*_k}U_{i^*_k, k})/(U_{i,k} + U_{i^*_k, k}) \text{ for } i\neq i^*_k \text{ and } \varepsilon_{i, k} = \infty \text{ for } i =  i^*_k.
  \end{equation*}
\end{theorem}

\citet{gabillon2013adaptive} analyzes a general adaptive submodular function maximization in
bandit setting. We have a regret bound in similar
form as~\eqref{eq:adaptive_regret_bound_a} if we directly apply Theorem 1
in~\cite{gabillon2013adaptive}. However, their version of $\Delta^{(k)}_{\max}$ is an
upper bound on the expected reward of policy $\pi^g$ from $k$ steps forward,
which is larger than our $\Delta^{(k)}_{\max}$. Their version of
$\Delta^{i,k}_{\min}$ is the minimum $(\mu_ic_i(\psi) - \min_{j\in
[m]}\mu_{j}c_{j}(\psi))/c_i(\psi)$ for all partial realization $\psi$ obtained
after policy $\pi^g$ is executed for $k$ steps, which is smaller than our
$\Delta^{i, k}_{\min}$. Our regret analysis is based on counting how many times
$\pi^g$ and $\pi^t$ choose different communities under the special full
realization $\phi$, while the analysis in~\cite{gabillon2013adaptive} is based on
counting how many times $\pi^g$ and $\pi^t$ choose different communities under
all possible full realizations.

\textbf{Discussion.} In this paper, we consider the online learning problem that
consists of $T$ rounds, and during each round, we explore the communities with a
budget $K$. Our goal is to maximize the {\it cumulative reward} in $T$ rounds.
Another important and natural setting is described as follows. We start to
explore communities with unknown sizes, and update the parameters every time we
explore the community for \textit{one step} (or for a few steps). Different from
the setting defined in this paper, here {\it a member will not contribute to the
reward if it has been met in previous rounds}. To differentiate the two
settings, let's call the latter one the ``{\it interactive community
exploration}'', while the former one the ``{\it repeated community
exploration}''. Both the repeated community exploration defined in this paper
and the interactive community exploration we will study as the future work have
corresponding applications. The former is suitable for online advertising where
in each round the advertiser promotes different products. Hence the rewards in
different rounds are additive. The latter corresponds to the adaptive online
advertising for the same product, and thus the rewards in different rounds are
dependent.

\section{Related Work}
\citet{golovin2011adaptive} show that a greedy policy could achieve
at least $(1 - 1/e)$ approximation for the adaptive submodular function. The
result could be applied to our offline adaptive problem, but by an independent
analysis we show the better result that the greedy policy is optimal.
Multi-armed bandit (MAB) problem is initiated by Robbins~\cite{robbins1985some}
and extensively studied in~\cite{berry1985bandit, sutton1998reinforcement,
bubeck2012regret}.
Our online learning algorithm is based on the extensively studied {\em
Upper Confidence Bound} approach~\cite{auer2002finite}. The non-adaptive
community exploration problem in the online setting fits into the general
combinatorial multi-armed bandit (CMAB) framework~\cite{gai2012combinatorial,
  kveton2015tight,CWYW16,chen2016combinatorial, wang2017improving}, 
where the reward is a set function of base arms. The CMAB problem is first
studied in~\cite{gai2012combinatorial}, and its regret bound is improved
by~\cite{CWYW16,kveton2015tight}. We leverage the analysis
framework in~\cite{CWYW16,wang2017improving} and prove a tighter
bound for our algorithm. \citet{gabillon2013adaptive} define an
adaptive submodular maximization problem in bandit setting. Our online adaptive
exploration problem is a instance of the problem defined
in~\cite{gabillon2013adaptive}. We prove a tighter bound than the one
in~\cite{gabillon2013adaptive} by using the properties of our problem.

Our model bears similarities to the optimal discovery problem proposed
in~\cite{bubeck2013optimal} such as we both have disjoint assumption, and both
try to maximize the number of target elements. However, there are also some
differences: (a) We use different estimators for our critical parameters,
because our problem setting is different. (b) Their online model is closer to
the interactive community exploration we explained in~\ref{sec:online_adaptive}
, while our online model is on repeated community exploration. As explained
in~\ref{sec:online_adaptive}, the two online models serve different applications
and have different algorithms and analyses. (c) We also have more comprehensive
studies on the offline cases.

\section{Future Work}
In this paper, we systematically study the community exploration problems. In
the offline setting, we propose the greedy methods for both of non-adaptive and
adaptive exploration problems. The optimality of the greedy methods are
rigorously proved. We also analyze the online setting where the community sizes
are unknown initially. We provide a CLCB algorithm for the online community
exploration. The algorithm has $O(\log T)$ regret bound. If we further allow the
full information feedback, the CLCB algorithm with some minor revisions has a
constant regret.

Our study opens up a number of possible future directions. For example, we can
consider various extensions to the problem model, such as more complicated
distributions of member meeting probabilities, overlapping communities, or even
graph structures between communities. We could also study the gap between
non-adaptive and adaptive solutions.

\subsubsection*{Acknowledgments}
We thank Jing Yu from School of Mathematical Sciences at Fudan University for
her insightful discussion on the offline problems, especially, we thank Jing Yu
for her method to find a good initial allocation, which leads to a faster greedy
method. Wei Chen is partially supported by the National Natural Science
Foundation of China (Grant No. 61433014). The work of John C.S. Lui is supported
in part by the GRF Grant 14208816.

\bibliographystyle{plainnat}
\bibliography{ref}

\begin{appendices}
\clearpage
\let\clearpage\relax
\noindent{\large\bfseries Supplementary Materials\par}

\section{Improved Budget Allocation Algorithm}\label{app:improved_budget_allocation}
\optimalitya*
\begin{proof}
 Let $r_i(j) = \mathbb{E}_{\Phi}[\abs{\{\Phi(i,1), \ldots, \Phi(i, j)\} }]
 	= d_i (1 - (1 - 1/d_i)^{j})$ denote the
 expected reward when the community $i$ is explored $j$ times. 
 Then we have that the marginal gain $r_i(j + 1) - r_i(j) = (1 - \mu_i)^j$ . 
 Define a matrix $\bm{X}\in \mathbb{R}^{m\times K}$, where the $(i,j)$-th
 entry $X_{i,j}$ is $(1-\mu_i)^{j-1}$. When the budget allocation is $\bm{k} =
 (k_1, \dots, k_m)$, the expected reward $r_{\bm{k}}(\bm{\mu})$ can be written as
 the sum of elements in $\bm{X}$, i.e., $r_{\bm{k}}(\bm{\mu})
 = \sum_{i=1}^{m}\sum_{j=1}^{k_i}X_{i,j}$. A key property of $\bm{X}$ is that the
 value in each row is decreasing with respect to the column index $j$. 
 Hence, for every $s \ge 1$, the $s$-th step of the greedy method chooses the $s$-th
 largest value in $\bm{X}$. 
 At step $s = K$, the greedy method finds
 the largest $K$ values in matrix $\bm{X}$. We can conclude that the greedy
 method obtains a budget allocation that maximizes the reward
 $r_{\bm{k}}(\bm{\mu})$.
\end{proof}

We propose a budget allocation algorithm which has time complexity $O(m\log m)$ in
Algo.~\ref{algo:improved_budget_allocation}. The basic idea is to find a good
initial allocation that is not far from the optimal allocation. Then starting
from the initial allocation, we run our original greedy method.
\begin{algorithm}[h]
\caption{Budget allocation algorithm}\label{algo:improved_budget_allocation}
  \begin{algorithmic}[1]
    \Require{parameters $\bm{\mu}$, budget $K > m$}
    \State{For $i\in [m]$, $k_i = \ceil{((K - m)/ \ln (1 - \mu_i)) / (\sum_{j=1}^m 1 / \ln
      (1 - \mu_j))}$}\label{line:initial_allocation}\Comment{A good initial allocation}
  \While{$\sum_{i=1}^{m}k_i < k$}
  \State{$i^*\leftarrow \argmax_{i} (1 - \mu_i)^{k_i}$}\Comment{$O(\log m)$ via
    using priority queue}
  \State{$k_{i^*}\leftarrow k_{i^*} + 1$}
    \EndWhile
  \end{algorithmic}
\end{algorithm}

\begin{lemma}[Basic property of optimal allocation]\label{lemma:basic_property}
  Let $\bm{k}^*$ be the optimal budget allocation when the parameter of the
 community is $\bm{\mu}$. For $i, j\in [m]$ , we have
 \begin{equation*}
  (1 - \mu_i)^{(k^*_i - 1)} \geq  (1 - \mu_j)^{k^*_j}.
 \end{equation*}
\end{lemma}
\begin{proof}
We define budget allocation $\bm{k}^{\prime}$ which is the same as $\bm{k}^{*}$
except that $k^{\prime}_i = k^{*}_i - 1$ and $k^{\prime}_j = k^*_j + 1$. 
If $(1 - \mu_i)^{(k^*_i - 1)} <  (1 - \mu_j)^{k^*_j}$ and $i \neq j$, then we
have
\begin{equation*}
 r_{\bm{k}^{\prime}}(\bm{\mu}) = r_{\bm{k}^*}(\bm{\mu}) -  (1 - \mu_i)^{(k^*_i - 1)} + (1 - \mu_j)^{k^*_j} > r_{\bm{k}^*}(\bm{\mu}),
\end{equation*}
which is contradict with the fact that $\bm{k}^*$ is the optimal solution. This
proves the lemma.
\end{proof}

\begin{lemma}[Allocation lower bound]\label{lemma:allocation_lower_bound}
 Let $\bm{k}^*$ be the optimal budget allocation when the parameter of the
 communities is $\bm{\mu}$. Define $\bm{k}^{-} = (k^{-}_1, \dots, k^-_m)$ where
 \begin{equation*}
   k^-_i = \frac{(K - m) / \ln (1 - \mu_i)}{\sum_{j=1}^m 1 / \ln (1 - \mu_j)}.
 \end{equation*}
 We have $k_i^* \geq k^-_i$.
\end{lemma}
\begin{proof}
 According to the definition of $\bm{k}^-$,  we have $k^-_i\ln (1 - \mu_i) =
 k^-_j\ln (1 - \mu_j)$ for $i, j \in [m]$.
 If  we can find $i$ such that $k^-_i  + 1 \leq k^*_i $, then 
 \begin{equation*}
  (1-\mu_j)^{k^-_j}  = (1-\mu_i)^{k^-_i} \geq (1-\mu_i)^{k^*_i - 1} \geq (1 - \mu_j)^{k^*_j}.
 \end{equation*}
Hence $k^-_j \leq k^*_j$. On the other hand, we can always find $k^-_i  + 1 \leq k^*_i$ since
$\sum_{i=1}^{m} (k^-_i + 1) = K$.
\end{proof}

In Algo.~\ref{algo:improved_budget_allocation}, we start with the lower bound
$\bm{k}^{-}$ of the optimal allocation. Since $\sum_{i=1}^{m}k^{-}_i = K - m$,
we have $\sum_{i=1}^{m}\abs{\ceil{k^{-}_i} - k^*_i}
\leq\sum_{i=1}^{m}\abs{k^{-}_i - k^*_i} = m$, which indicates
Algo.~\ref{algo:improved_budget_allocation} obtains the optimal budget
allocation within $m$ steps. We also provide an upper bound $\bm{k}^+$ in the
following. The upper bound is also close to the optimal budget since
$\sum_{i=1}^m \abs{\floor{k^+_i} - k^*_i} \leq \sum_{i=1}^m \abs{k^+_i - k^*_i}
= m$.

\begin{lemma}[Allocation upper bound]\label{lemma:allocation_upper_bound}
 Let $\bm{k}^*$ be the optimal budget allocation when the parameter of the
 communities is $\bm{\mu}$. Define $\bm{k}^{+} = (k^{+}_1, \dots, k^+_m)$ where
 \begin{equation*}
   k^+_i = \frac{K / \ln (1 - \mu_i)}{\sum_{j=1}^m 1 / \ln (1 - \mu_j)} + 1.
 \end{equation*}
 We have $k_i^* \leq k^+_i$.
\end{lemma}
\begin{proof}
 According to the definition of $\bm{k}^+$,  we have $(k^+_i - 1)\ln (1 - \mu_i) =
 (k^+_j -1)\ln (1 - \mu_j)$ for $i, j \in [m]$.
 If  we can find $i$ such that $k^+_i  - 1 \geq k^*_i $, then 
 \begin{equation*}
  (1-\mu_j)^{k^+_j - 1}  = (1-\mu_i)^{k^+_i - 1} \leq (1-\mu_i)^{k^*_i} \leq (1 - \mu_j)^{k^*_j-1}.
 \end{equation*}
Hence $k^+_j \geq k^*_j$. On the other hand, we can always find $k^+_i  - 1 \geq k^*_i$ since
$\sum_{i=1}^{m} (k^+_i - 1) = K$.
\end{proof}

\section{Properties of Greedy Policy}
In the following, we show some important properties of the greedy policy. We
further extend the definition of reward with a non-decreasing function $f$ as $
R(\bm{k}, \phi)
=f\left(\sum_{i=1}^{m}\abs{\bigcup_{\tau=1}^{k_i}\{\phi(i,\tau)\}}\right)$.

\subsection{Optimality of greedy policy}\label{app:optimality_greedy_policy}
In this part, we prove that the greedy policy is the optimal policy for our
adaptive community exploration problem. To prove the optimality, we first
rewrite the proof sketch of Theorem~\ref{thm:greedy_policy_is_optimal}, and then
provide the supporting
Lemma~\ref{lemma:property_greedy_policy}\&\ref{lemma:concantenation_policy}.

\optimalityb*
\begin{proof}
Let $F_{\pi}(\psi, t)$ denote the expected {\em marginal gain} when we further
explore communities for $t$ steps with policy $\pi$ starting from a partial
realization $\psi$. We want to prove that for all $\psi$, $t$ and $\pi$,
$F_{\pi^g} (\psi,t)\ge F_\pi(\psi,t)$, where $\pi^g$ is the greedy policy and
$\pi$ is an arbitrary policy. If so, we simply take $\psi = \emptyset$, and
$F_{\pi^g} (\emptyset,t)\ge F_\pi(\emptyset,t)$ for every $\pi$ and $t$ exactly
shows that $\pi^g$ is optimal. We prove the above result by an induction on $t$.
Recall that $c_i(\psi)$ is the number of distinct members met in community $C_i$
under the partial realization $\psi$. Define $c(\psi) = \sum_i c_i(\psi)$ and
$\Delta_{\psi, f} = f(c(\psi) + 1) - f(c(\psi))$.

For all $\psi$ and $\pi$, we first claim that $F_{\pi^g} (\psi,1)\ge
F_\pi(\psi,1)$ holds. Suppose that policy $\pi$ chooses community $C_i$ to
explore based on the partial realization $\psi$. Since the exploration will
return a new member with probability $1 - \mu_ic_i(\psi)$, the expected marginal
gain $F_{\pi}(\psi, 1)$ is $(1 - \mu_ic_i(\psi))[f(c(\psi) + 1) - f(c(\psi))]$.
Note that the greedy policy $\pi^g$ chooses community $C_{i^*}$ to explore with
$i^*\in \argmax_{j} (1 - \mu_j c_j(\psi))$, and $\Delta_{\psi, f}$ does not
depend on the policy. Hence, $F_{\pi^g} (\psi,1)\ge F_\pi(\psi,1)$.

Assume $F_{\pi^g} (\psi,t')\ge F_\pi(\psi,t')$ holds for all $\psi$, $\pi$, and
$t'\le t$. Our goal is to prove that $F_{\pi^g} (\psi,t+1)\ge F_\pi(\psi,t+1)$.
Suppose that in the first step after $\psi$, policy $\pi$ chooses $C_i$ to
explore based on partial realization $\psi$, and let $\pi(\psi)=(i,\tau)$.
Define $E_{\psi}$ as the event that the member $\Phi(i,\tau)$ is not met in
partial realization $\psi$, for $\Phi\sim \psi$. In the following, we represent
partial realization $\psi$ equivalently as a relation $\{((i,\tau),
\psi(i,\tau)) \mid (i,\tau) \in \dom(\psi) \}$, so we could use
$\psi\cup\{((i,\tau), \Phi(i,\tau))\}$ to represent the new partial realization
extended from $\psi$ by one step with $(i,\tau)$ added to the domain and
$\Phi(i,\tau)$ as the member met for this exploration of $C_i$. Then we have
 \begin{align*}
   &F_{\pi}(\psi, t + 1) = \sum\nolimits_{v\in C_i}\Pr\left( \Phi(i,\tau) = v\right)\mathbb{E}_{\Phi}[F_{\pi}(\psi, t + 1) \mid \Phi\sim \psi, \Phi(i,\tau) = v]\\
   &= \sum_{v\in C_i}\mu_i\mathbb{E}_{\Phi}[F_{\pi}(\psi\cup\{((i,\tau), \Phi(i,\tau))\}, t) + f(c(\psi) + \bOne\{E_{\psi}\}) - f(c(\psi)) \mid \Phi\sim \psi, \Phi(i,\tau) = v]\\
   &\leq\sum\nolimits_{v\in C_i} \mu_i\mathbb{E}_{\Phi}[F_{\pi^g}(\psi\cup\{((i,\tau), \Phi(i,\tau))\}, t) \mid \Phi\sim \psi, \Phi(i,\tau) = v] + (1 - \mu_ic_i(\psi))\Delta_{\psi,f}.
 \end{align*}
 The $2$nd line above is derived directly from the definition of
 $F_{\pi}(\psi,t)$. 
 The $3$rd line is based on the induction hypothesis that
 $F_{\pi}(\psi', t) \leq F_{\pi^g}(\psi', t)$ holds for all $\psi'$.
 An important observation is that $F_{\pi^g}(\psi, t)$ has
 equal value for any partial realization $\psi$ associated with the same status
 $\bm{s}$ since the status is enough for the greedy policy to determine the
 choice of next community. Formally, we define $F_{g}(\bm{s}, t) =
 F_{\pi^g}(\psi, t)$ for any partial realization that satisfies $\bm{s} = (1 -
 c_1(\psi)/d_1,\dots, 1 - c_m(\psi)/d_m)$. Let $C_{i^*}$ denote the
 community chosen by policy $\pi^g$ under realization $\psi$, i.e., $i^*\in
 \argmax_{i\in [m]} 1 - c_i(\psi)\mu_i$. 
 Let $\bI_i$ be the $m$-dimensional unit vector with $1$ in the $i$-th entry and $0$ in all other entries.
 Therefore,
 \begin{align*}
  F_{\pi}(\psi, t + 1) &\leq c_i(\psi) \cdot \mu_i F_{g}(\bm{s}, t) + (d_i - c_i(\psi)) \cdot \mu_i F_{g}(\bm{s} - \mu_i \bI_i, t) + (1 - \mu_ic_i(\psi))\Delta_{\psi, f} \\
   & \leq \mu_{i^*}c_{i^*}(\psi)F_{g}(\bm{s}, t) + (1 - \mu_{i^*}c_{i^*}(\psi))F_{g}(\bm{s} - \mu_{i^*} \bI_{i^*}, t)  + (1 - \mu_{i^*}c_{i^*}(\psi))\Delta_{\psi, f} \OnlyInFull{\tag{Lemma~\ref{lemma:concantenation_policy}}}\\
   & = F_{g}(\bm{s}, t + 1) = F_{\pi^g}(\psi, t + 1). \OnlyInFull{\tag{Lemma~\ref{lemma:property_greedy_policy}}}
 \end{align*}
 The key is to prove the correctness of Line 2 in above equation. It indicates
 that if we choose a sub-optimal community at first, and then we switch back to
 the greedy policy, the expected reward would be smaller. The proof is nontrivial
 and relies on a careful analysis based on the stochastic transitions among
 status vectors. The above result completes the induction step for $t+1$. Thus
 the theorem holds.
\end{proof}

\begin{lemma}\label{lemma:property_greedy_policy}
 Let $\bm{s} = (s_1,\dots, s_m)$ be a status where each entry $s_i\in [0, 1]$.
 We have
 \begin{equation*}
   F_{g}(\bm{s}, t + 1) = (1 - s_{i^*})F_{g}(\bm{s}, t) + s_{i^*}F_{g}(\bm{s} - \mu_{i^*} \bI_{i^*}, t)  + s_{i^*} (f(c(\psi) + 1) - f(c(\psi))),
 \end{equation*}
where $i^* = \argmax_{i\in [m]}s_i$. Here $\psi$ is any partial realization
corresponding to status $\bm{s}$.
\end{lemma}
\begin{proof}
 For any partial realization $\psi$ associated with status $\bm{s}$, $\pi^g$
 would choose community $i^*$. With probability $\mu_{i^*}c_{i^*}(\psi) = 1 -
 s_{i^*}$, we will obtain a member that is already met. If so, the communities
 stay at the same status. Hence, with probability $1 -s_{i^*}$, the expected
 extra reward is $F_g(\bm{s}, t)$ after the first step exploration. With
 probability $1 -\mu_{i^*}c_{i^*}(\psi) = s_{i^*}$, we will obtain an unseen
 member in $C_{i^*}$. The communities will transit to next status $\bm{s} -
 \mu_{i^*}\bI_{i^*}$. Therefore, with probability $s_{i^*}$, the expected extra
 reward is $F_g(\bm{s} - \bI_{i^*}, t) + f(c(\psi) + 1) f(c(\psi))$ after
 the first step exploration.
\end{proof}

\begin{lemma}\label{lemma:concantenation_policy}
 Let $\bm{s} = (s_1,\dots, s_m)$ be a status where each entry $s_i\in [0, 1]$
 and $\psi$ be any partial realization corresponding to $\bm{s}$. 
 We have 
 \begin{equation}\label{eq:condition2}
   \begin{split}
    &(1 - s_i)F_{g}(\bm{s}, t) + s_i F_{g}(\bm{s} - \mu_i \bI_i, t) + s_i\Delta_c \\
    \leq\ \ &(1 - s_{i^*})F_{g}(\bm{s}, t) + s_{i^*}F_{g}(\bm{s} - \mu_{i^*} \bI_{i^*}, t)  + s_{i^*}\Delta_c,
   \end{split}
 \end{equation}
where $i^* \in \argmax_{i\in [m]}s_i$, $s_i < s_{i^*}$ and $\Delta_c = f(c(\psi) +
1) - f(c(\psi))$.
\end{lemma}
\begin{proof}
  Let $A(\bm{s}, i, t)$ denote the first line of Eq.~\eqref{eq:condition2}, i.e.,
  \begin{equation*}
    A(\bm{s}, i, t) = (1 - s_i)F_{g}(\bm{s}, t) + s_i F_{g}(\bm{s} - \mu_i \bI_i, t) + s_i\Delta_c. 
  \end{equation*}
  Note that $A(\bm{s}, i, t)$ is the expected reward of the following adaptive process.
  \begin{enumerate}
  \item At the first step, choose an arbitrary community $C_i$ (different from
    $C_{i^*}$) to explore.
  \item From the second step to the $(t + 1)$-th step, explore communities with
    the greedy policy $\pi^g$.
  \end{enumerate}
  Similarly, $A(\bm{s}, i^*, t)$ is the expected reward of the $t+1$ step community
  exploration via the greedy policy, i.e., $A(\bm{s}, i^*, t) = F_g(\bm{s}, t + 1)$. 
  Eq.~\eqref{eq:condition2} can be written as $A(\bm{s}, i, t) \leq F_g(\bm{s}, t + 1)$. We
  prove this inequality by induction. When $t = 0$, we have $A(\bm{s}, i, t) =
  s_i\Delta_c$, and $A(\bm{s}, i^*, t) = s_{i^*}\Delta_c$. Hence, $A(\bm{s}, i,
  t)\leq A(\bm{s}, i^*, t) = F_g(\bm{s}, t + 1)$ when $t = 0$.
  Assume $A(\bm{s}, i, t^{\prime}) \leq F_g(\bm{s},  t^{\prime} + 1)$ holds for
  any $0\leq t^{\prime} \leq t$, and any status $\bm{s}$.
  Our goal is to prove that $A(\bm{s}, i, t + 1) \leq
  A(\bm{s}, i^*, t+1) = F_g(\bm{s}, t + 2)$.
  We expand $A(\bm{s}, i, t + 1)$ as follows.      
  \begin{align*}
    A(\bm{s}, i, t + 1) &= (1 - s_i)F_{g}(\bm{s}, t + 1) + s_i F_{g}(\bm{s} - \mu_i \bI_i, t + 1) + s_i\Delta_c\\
                 &= (1 - s_i)\left(  (1 - s_{i^*})F_{g}(\bm{s}, t) + s_{i^*}F_{g}(\bm{s} - \mu_{i^*} \bI_{i^*}, t)  + s_{i^*}\Delta_{c} \right)\\
                & + s_i ((1 - s_{i^*})F_g(\bm{s} -\mu_i\bm{I}_{i}, t) + s_{i^*}F_g(\bm{s} - \mu_i\bm{I}_i - \mu_{i^*}\bm{I}_{i^*}, t) + s_{i^*}\Delta_{c + 1}) \\
                & + s_i\Delta_c.
  \end{align*}
Here $\Delta_{c+1} = f(c(\psi) + 2) - f(c(\psi) + 1)$. Above expansion of $A(i,
t+1)$ is based on Lemma~\ref{lemma:property_greedy_policy}. We expand $A(\bm{s},
i^*, t+1)$ as follows.
\begin{align*}
 A(\bm{s}, i^*, t + 1) &=  (1 - s_{i^*})F_{g}(\bm{s}, t + 1) + s_{i^*}F_{g}(\bm{s} - \mu_{i^*} \bI_{i^*}, t + 1)  + s_{i^*}\Delta_c \\
               & \geq (1 - s_{i^*})\left( (1 - s_{i})F_{g}(\bm{s}, t) + s_{i}F_{g}(\bm{s} - \mu_{i} \bI_{i}, t)  + s_{i}\Delta_{c} \right)\tag{assumption $A(\bm{s}, i, t)\leq F_g(\bm{s}, t + 1)$}\\
               & + s_{i^*} ((1 - s_i)F_g(\bm{s} - \mu_{i^*}\bm{I}_{i^*}, t) + s_{i}F_g(\bm{s} -\mu_{i^*}\bm{I}_{i^*} - \mu_i\bm{I}_i, t) + s_{i}\Delta_{c+1}) \tag{assumption $A(\bm{s} - \mu_{i^*}\bm{I}_{i^*}, i, t) \leq F_g(\bm{s} - \mu_{i^*}\bm{I}_{i^*}, t + 1)$}\\
               & + s_{i^*} \Delta_{c}\\
  &= A(i, t + 1).
\end{align*}
This completes the proof.
\end{proof}

\textbf{Remarks.} During the rebuttal of this paper, we realized
that~\citet{bubeck2013optimal} applied similar inductive reasoning techniques to
prove the optimality of the greedy policy for their optimal discovery problem
(Lemma 2 of~\cite{bubeck2013optimal}). To quantitatively measure how good is the
greedy policy, we also give a formula to show the exact difference between $A(\bm{s},
i, t)$ and $A(\bm{s}, i^*, t)$ in Sec.~\ref{sec:exact_reward_gap}.

\subsection{Computation of expected reward}\label{app:expected_reward_greedy_policy}
Lemma~\ref{lemma:property_greedy_policy} indicates $r_{\pi^g}(\bm{\mu})$ can be
computed in a recursive way. However, the recursive method has time complexity
$O(2^K)$. It is impractical when $K$ is large. In the following we show that the
expected reward of policy $\pi^g$ can be computed in polynomial time.

\subsubsection{Transition probability list of greedy policy}\label{app:transition_probability_list_greedy_policy}
Assume we explore the communities via the greedy policy when the communities
already have partial realization $\psi$. Define $s_{i, 0} = 1 - \mu_ic_i(\psi)$
and $\bm{s}_0 = (s_{1, 0}, \dots, s_{m, 0})$. The greedy policy will choose
community $i^*_0$ to explore, where $i^*_0 \in \argmax_{i}s_{i, 0}$. After
one step exploration, the communities stay at the same status $\bm{s}_0$ with
probability $q_0 \defeq 1 - s_{i^*_0}$. The communities transit to next status
$\bm{s}_1 \defeq \bm{s}_0 -\mu_{i^*_0}\bI_{i^*_0}$ with probability $p_0 \defeq
s_{i^*_0}$. We recursively define $\bm{s}_{t+1}$ as $\bm{s}_{t} -
\mu_{i^*_{t}}\bI_{i^*_{t}}$, where $i^*_t \in \argmax_{i}s_{i, t}$. We call $p_t
\defeq \max_i s_{i,t}$ the {\em transition probability} and $q_t \defeq 1 -
p_t$ the {\em loop probability}. Each time the communities transit to next
status, a new member will be met. During the exploration, the number of
different statuses the communities can stay is at most $1 + \sum_id_i -
c_i(\psi)$ since there are $D \defeq\sum_id_i - c_i(\psi)$ unseen members in
total. Based on above discussion, we define a {\em transition probability list}
$\mathcal{P}(\pi^g, \psi) \defeq (p_0, \dots, p_D)$, where $p_D \equiv 0$. The list
$\mathcal{P}(\pi^g, \psi)$ is unique for any initial partial realization $\psi$.
Fig.~\ref{fig:illustration_statuses} gives an example to demonstrate statuses
and the list $\mathcal{P}(\pi^g, \psi)$. 

\begin{corollary}\label{corollary:observation_on_probability_list}
 Let $\psi$ be any partial realization corresponding to the status $\bm{s} =
 (s_1,\dots, s_m)$. The number of unseen members $\sum_id_i - c_i(\psi)$
 is denoted as $D$. The probability list $\mathcal{P}(\pi^g, \psi) = (p_0,\dots,
 p_{D})$ can be
 obtained by sorting $\cup_{i=1}^m\{s_i, s_i - \mu_i, \dots,
 \mu_i\}\cup \{0\}$ in descending order. 
\end{corollary}
Corollary~\ref{corollary:observation_on_probability_list} is an important
observation based on the definition of transition probability list.

\begin{figure}[htbp]
  \centering
 \includegraphics[width=0.9\textwidth]{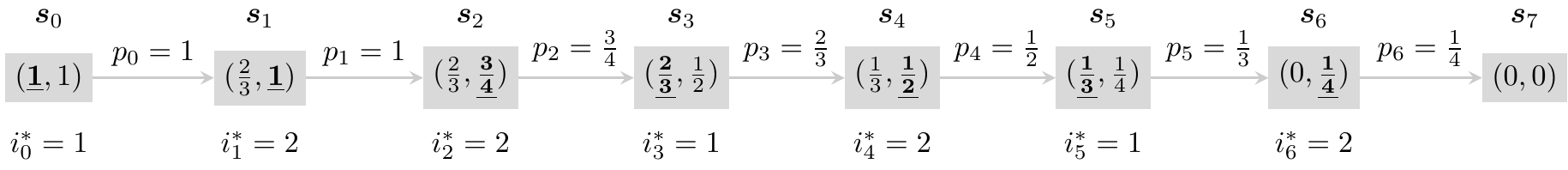} 
 \caption{Illustration with $\bm{d} = (3, 4)$ and empty partial realization. The
   initial status is $(1, 1)$. The list $\mathcal{P}(\pi^g, \emptyset) = (1, 1, 3/4,
 2/3, 1/2, 1/3, 1/4, 0)$.}\label{fig:illustration_statuses}
\end{figure}

\subsubsection{Compute the expected reward efficiently}
\begin{lemma}\label{lemma:expected_reward_probability_list}
  Let $\psi$ be a partial realization and $\bm{s}_0$ be the corresponding
  status. The number of unseen members is denoted
  as $D = \sum_i d_i - c_i(\psi)$. The transition probability list
  is $\mathcal{P}(\pi^g, \psi) = (p_0, \dots, p_{D})$. Then 
  \begin{equation*}
   F_{\pi^g}(\psi, t) = F_g(\bm{s}_0, t) = \sum_{j=0}^{\min\{t, D\}} (f(j + c(\psi)) - f(c(\psi)))\times\left(\Pi_{l=0}^{j-1}p_j \right) \times \left(\sum_{I\in \mathcal{I}(j, t - j)} \Pi_{l\in I} q_l\right), 
  \end{equation*}
  where $q_l = 1 - p_l$ and $\mathcal{I}(j, t-j)$ consists of
  subsets of multi-set $\{0, \dots, j\}^{t-j}$ with fixed size $t-j$.
\end{lemma}
\begin{proof}
  When the communities ends at status $\bm{s}_j$, we meet $j$ distinct
  members. Let $\Pr(\bm{s}_j\square)$ be the probability for this event.
  We can the {\em transition step} as the communities transit to a new status,
  and the {\em loop step} as the communities stay at the same status.
  When the communities ends at status $\bm{s}_j$, we have $j$ transition steps
  and $t - j$ loop steps. The communities takes loops at statuses $\{\bm{s}_0,
  \dots, \bm{s}_j\}$. Hence, 
  \begin{equation*}
   \Pr(\bm{s}_j\square) = \sum_{I\in \mathcal{I}(j, t-j)}  \Pi_{l=0}^{j-1}p_j \cdot \Pi_{l\in I} (1 - p_l) = \Pi_{l=0}^{j-1}p_j \times \sum_{I\in \mathcal{I}(j, t-j)} \Pi_{l\in I} q_l.
  \end{equation*}
  The reward $F_{\pi^g}(\psi, t) = \sum_{j=1}^{\min\{t, D\}} (f(j + c(\psi)) - f(c(\psi)))\times
  \Pr(\bm{s}_j\square)$.
\end{proof}
For later analysis, we define the {\em loop probability}
\begin{equation*}
 L(\{q_0, \dots, q_j\}, t) \defeq \sum_{I\in \mathcal{I}(j, t)} \Pi_{l\in I} q_l
\end{equation*}
since $\sum_{I\in \mathcal{I}(j, t)} \Pi_{l\in I} q_l$ is just a function of
$\{q_0, \dots, q_j\}$ and $t$ ($t \geq 1$). Actually, $L(\{q_0, \dots, q_j\},
t)$ aggregates the product of all possible $t$ elements in $\{q_0,\ldots,
q_{j}\}$. Note that each element in $\{q_0,\ldots, q_j\}$ can be chosen multiple
times. W.l.o.g, we define $L(\{q_0, \dots, q_j\}, t) = 1$ and $\Pi_{l=0}^{t-1} p_l=
1$ when $t = 0$. Based on the
definition, we can write $L(\{q_0, \dots, q_j\}, t)$ in a recursive way as
follows.
\begin{equation}
  \label{eq:recursive_loop_probability}
  L(\{q_0, \dots, q_j\}, t) = \sum_{s=0}^{t} q^s_aL(\{q_0, \dots, q_j\}\backslash\{q_a\}, t - s).
\end{equation}
Here $a\in \{0, \dots, j\}$. According to
Eq.~\ref{eq:recursive_loop_probability}, the probability $\sum_{I\in
\mathcal{I}(j, t - j)} \Pi_{l\in I} q_l$ can be computed in $O((t - j)j^2)$ via
{\em dynamic programming}. Hence $r_{\pi^g}(\bm{\mu}) = F_g((1,\dots, 1), K)$
can be computed in $O(K\min\{K, D\}^2)$ according to
Lemma~\ref{lemma:expected_reward_probability_list}.

\subsection{Reward gap between optimal policy and sub-optimal
policy}\label{sec:exact_reward_gap}
Recall that $A(\bm{s}, i, t)$ is the expected reward of the following adaptive
process.
\begin{enumerate}
  \item At the first step, choose an arbitrary community $C_i$ (different from
   $C_{i^*}$) to explore.
  \item From the second step to the $(t + 1)$-th step, explore communities with
    the greedy policy $\pi^g$.
\end{enumerate}
Here $\bm{s}$ is the initial status of the communities.
Lemma~\ref{lemma:concantenation_policy} only proves that $A(\bm{s}, i, t) \leq
F_g(\bm{s}, t + 1)$. In the following, we aim to answer the following question:
\begin{itemize}
\item  How much is $F_g(\bm{s, t + 1})$ larger than $A(\bm{s}, i, t)$?
\end{itemize}

\subsubsection{Analysis of loop probability}\label{app:loop_probability}
The following two corollaries show the basic properties of
the {\em loop probability}. 

\begin{corollary}\label{corollary:sum_of_probability}
 For a transition probability list $\mathcal{P}(\pi^g, \psi) = (p_0, \dots,
 p_{D})$, we have  
 \begin{equation*}
  \sum_{j=0}^Mp_0\times\dots\times p_{j-1}\times L(\{q_0,\dots,q_j\}, t - j) = 1,
 \end{equation*}
 where $q_j = 1 - p_j$ and $M = \min\{t, D\}$.
\end{corollary}
Corollary~\ref{corollary:sum_of_probability} says the probabilities that the communities ends at status
$\{\bm{s}_0, \ldots, \bm{s}_D\}$ sums up to 1. 

\begin{corollary}\label{corollary:loop_probabilities_switch}
  For a transition probability list $\mathcal{P}(\pi^g, \psi)  = (p_0,\dots, p_D)$ and
  $a, b\in \{0, \dots, j\}$ ($j\leq D, t\geq 1$), we have
  \begin{align*}
   & L(\{q_0, \dots, q_{j}\}\backslash\{q_a\}, t) - L(\{q_0, \dots, q_{j}\}\backslash\{q_{b}\}, t) = (q_b - q_a)L(\{q_0, \dots, q_{j}\}, t-1),
\end{align*}
where $D = \sum_i d_i - c_i(\psi)$ and $q_j = 1 - p_j$.
\end{corollary}
\begin{proof}
    We prove the corollary according to Eq.~\eqref{eq:recursive_loop_probability}.
   \begin{align*}
   & L(\{q_0, \dots, q_{j}\}\backslash\{q_a\}, t) - L(\{q_0, \dots, q_{j}\}\backslash\{q_{b}\}, t)\\
   =&  \sum_{s=0}^{t}(q^s_b - q^s_a) L(\{q_0, \dots, q_{j}\}\backslash\{q_a, q_b\}, t - s)\tag{by Eq.~\eqref{eq:recursive_loop_probability}}\\
    =&  \sum_{s=0}^{t-1}(q^{s+1}_{b} - q^{s+1}_a) L(\{q_0, \dots, q_{j}\}\backslash\{q_a, q_b\}, t - s - 1)\tag{replace $s-1$ as $s^{\prime}$}\\
    =&  (q_{b} - q_a)\sum_{s=0}^{t-1}\sum_{m=0}^{s}q^{s - m}_{b}q^{m}_{a}L(\{q_0, \dots, q_{j}\}\backslash\{q_a, q_{b}\}, t - 1 - s)\tag{sum of geometric sequence}\\
    = & (q_{b} - q_a)L(\{q_0, \dots, q_{j}\}, t -1).\tag{by definition or expanding Eq.~\eqref{eq:recursive_loop_probability}}
\end{align*}
This completes the proof.
\end{proof}

\subsubsection{Pseudo reward}
\begin{lemma}\label{lemma:small_gap}
    For a transition probability list $\mathcal{P}(\pi^g, \psi) = (p_0,\dots,
    p_D)$ and a non-decreasing function $f(x)$,
    a pseudo reward $R(k)$ is defined as
  \begin{align*}
    R(k)
    & = q_k\sum_{j=0}^Mf(j)\times p_0\times\dots\times p_{j-1}\times L(\{q_0,\dots, q_j\}, t - j)\\
    & + p_k \sum_{j=0}^{k-1}f(j + 1)\times p_0\times \dots\times p_{j-1} \times L(\{q_0, \dots, q_j\}, t - j) \\
    & + \sum_{j=k}^{M^{\prime}}f(j + 1)\times p_0\times \dots\times p_j \times L(\{q_0, \dots, q_{j+1}\}\backslash\{q_k\}, t - j),
  \end{align*}
  where $M = \min\{D, t\}$ and $M^{\prime} = \{D - 1, t\}$.
  We claim that for $0\leq k \leq M - 1$,
\begin{align*}
  R(k) - R(k+1)\!=\!(p_k\!-\!p_{k+1}) \left(\sum_{j=0}^{k}(f(j + 1) - f(j))p_0\times \dots\times p_{j-1} \times L(\{q_0, \dots, q_j\}, t - j)\right).
\end{align*}
\end{lemma}
\begin{proof}
  We expand $R(k) - R(k+1)$ as follows using the definition.
\begin{align*}
    & R(k) - R(k + 1)\\
      & = -(p_k - p_{k+1})\sum_{j=0}^Mf(j)\times p_0\times\dots\times p_{j-1}\times L(\{q_0,\dots, q_j\}, t - j) \\
    & + (p_k - p_{k+1}) \sum_{j=0}^{k-1}f(j + 1)\times p_0\times \dots\times p_{j-1} \times L(\{q_0, \dots, q_j\}, t - j) \\
    & + f(k + 1)\times p_0\times \dots\times p_{k-1} \times p_k \times L(\{q_0, \dots, q_{k+1}\}\backslash\{q_k\}, t - k) \tag{from $R(k)$}\\
    & - f(k + 1)\times p_0\times \dots\times p_{k-1} \times p_{k+1}\times L(\{q_0, \dots, q_k\}, t - k) \tag{from $R(k+1)$}\\
    & + \sum_{j=k+1}^{M^{\prime}}f(j + 1)\times p_0\times \dots\times p_j \times (L(\{q_0, \dots, q_{j+1}\}\backslash\{q_k\}, t - j) \\
    & \quad\underbrace{\quad\quad\quad\quad\quad\quad\quad\quad\quad\quad\quad\quad\quad- L(\{q_0, \dots, q_{j+1}\}\backslash\{q_{k+1}\}, t - j))}_{(p_{k} - p_{k+1})\sum_{j=k+1}^{M-1}f(j + 1)\times p_0\times \dots\times p_j\times L(\{q_0, \dots, q_{j+1}\}, t - j -1)}.
 \end{align*}
The last line of above equation can be rewritten with the
Corollary~\ref{corollary:loop_probabilities_switch}. The summation from $j =
k+1$ to $j = M-1$ in the last line cancels out with the second line when $j = k
+ 2$ to $j = M$. The summation from $j = 0$ to $j = k$ in the second line can be
combined with the third line. We continue the computation of $R(k) - R(k+1)$ by
rearranging its expansion.
\begin{align*}
   & \quad R(k) - R(k + 1)\\
  &\begin{aligned}
  = &  -(p_k - p_{k+1})f(k+1)\times p_0\times\dots\times p_{k}\times L(\{q_0,\dots, q_{k+1}\}, t - k - 1)\\ 
  & -(p_k - p_{k+1}) f(k + 1)\times p_0\times \dots\times p_{k-1} \times L(\{q_0, \dots, q_k\}, t - k) \\
  & +(p_k - p_{k+1}) \sum\nolimits_{j=0}^{k}(f(j + 1) - f(j))\times p_0\times \dots\times p_{j-1} \times L(\{q_0, \dots, q_j\}, t - j) \\
    & + f(k + 1)\times p_0\times \dots\times p_{k-1} \times p_k \times L(\{q_0, \dots, q_{k+1}\}\backslash\{q_k\}, t - k) \\
    & -f(k + 1)\times p_0\times \dots\times p_{k-1} \times p_{k+1}\times L(\{q_0, \dots, q_k\}, t - k). \\
  \end{aligned}
\end{align*}
Define $\Delta_k$ as the sum of the 2nd, 3rd, 5th, 6th line in above equation.
We have
\begin{align*}
   & R(k) - R(k + 1)\\
   & \begin{rcases}
    = &-(p_k - p_{k+1})\times f(k + 1)\times p_0\times\dots\times p_{k-1}\times L(\{q_0,\dots, q_k\}, t - k) \\ 
    &  -(p_k - p_{k+1})\times f(k + 1)\times p_0\times\dots\times p_{k}\times L(\{q_0,\dots, q_{k+1}\}, t - k - 1) \\ 
    & + f(k + 1)\times p_0\times \dots\times p_{k-1} \times p_k \times L(\{q_0, \dots, q_{k+1}\}\backslash\{q_k\}, t - k) \\
    & - f(k + 1)\times p_0\times \dots\times p_{k-1} \times p_{k+1}\times L(\{q_0, \dots, q_k\}, t - k)
  \end{rcases}\triangleq \Delta_k\\
  &\quad +(p_k - p_{k+1}) \sum_{j=0}^{k}(f(j + 1) - f(j))\times p_0\times \dots\times p_{j-1} \times L(\{q_0, \dots, q_j\}, t - j).
\end{align*}
We rewrite $\Delta_k$ as follows.
\begin{align*}
  \Delta_k / f(k + 1)= &\ p_0\times \dots\times p_{k-1} \times p_k \times L(\{q_0, \dots, q_{k+1}\}\backslash\{q_k\}, t - k) \\
  &\begin{rcases}
  & - p_0\times \dots\times p_{k-1} \times p_k \times L(\{q_0, \dots, q_{k}\}, t - k) \\
  &  + p_0\times \dots\times p_{k-1} \times p_k \times L(\{q_0, \dots, q_{k}\}, t - k) \\
  \end{rcases}\textit{cancel each other}\\
&-(p_k - p_{k+1})\times p_0\times\dots\times p_{k-1}\times L(\{q_0,\dots, q_k\}, t - k) \\ 
    &  -(p_k - p_{k+1})\times p_0\times\dots\times p_{k}\times L(\{q_0,\dots, q_{k+1}\}, t - k - 1) \\ 
    & - p_0\times \dots\times p_{k-1} \times p_{k+1}\times L(\{q_0, \dots, q_k\}, t - k)).
\end{align*}
According to Corollary~\ref{corollary:loop_probabilities_switch}, the first
line and the second line of above equation equals to $(p_k - p_{k+1})\times
p_0\times\dots\times p_{k}\times L(\{q_0,\dots, q_{k+1}\}, t
- k - 1)$, which cancels out with the fifth line.  
Hence, we have
\begin{align*}
  \Delta_k / f(k+1) &=  p_0\times \dots\times p_{k-1} \times p_k \times L(\{q_0, \dots, q_{k}\}, t - k) \\
 &-(p_k - p_{k+1})\times p_0\times\dots\times p_{k-1}\times L(\{q_0,\dots, q_k\}, t - k) \\ 
                    & - p_0\times \dots\times p_{k-1} \times p_{k+1}\times L(\{q_0, \dots, q_k\}, t - k))\\
 & =   0.
\end{align*}
With above result of $\Delta_k = 0$, we prove that
\begin{align*}
 & R(k) - R(k+1) \\
 & = (p_k - p_{k+1}) \left(\sum_{j=0}^{k}(f(j + 1) - f(j))p_0\times \dots\times p_{j-1} \times L(\{q_0, \dots, q_j\}, t - j)\right) \geq 0. 
\end{align*}
This completes the proof.
\end{proof}

\subsubsection{Reward gap} 
  Let $\psi, \psi^{\prime}, \psi^{\prime\prime}$ be any partial realization
  corresponding to the status $\bm{s}, \bm{s} - \mu_{i^*} \bI_{i^*}, \bm{s} -
  \mu_i \bI_i$ respectively. Define $\mathcal{P}(\pi^g, \psi) = (p_0,\ldots,
  p_D)$, where $D = \sum_id_i - c_i(\psi)$.
  Recalling Corollary~\ref{corollary:observation_on_probability_list}, we
  know that $(p_0,\dots, p_{D-1})$ can be obtained by sorting 
  $\cup_{i=1}^m\{s_i, s_i - \mu_i, \dots, \mu_i\}$. Assume the first time $s_i$
  appear in $(p_0, \dots, p_D)$ is the $k$-th entry, i.e., $k = \min\{k^{\prime}:
  0\leq k^{\prime} \leq D, p_{k^{\prime}} = s_i\}$. According to
  Corollary~\ref{corollary:observation_on_probability_list}, we have the following.
  \begin{align*}
  \mathcal{P}(\pi^g, \psi^{\prime}) &= (p_1,\dots, p_{D}),\\
  \mathcal{P}(\pi^g, \psi^{\prime\prime}) & = (p_0,\dots, p_{k-1}, p_{k+1}, \dots, p_D).
  \end{align*}
  Note that $p_0 = s_{i^*}$ and $p_k = s_{i}$.
  Let $M = \min\{D, t\}$, $M^{\prime} = \min\{D - 1, t\}$, and $f^{\prime}(j) =
  f(j + c(\psi)) - f(c(\psi))$. The second
  line of Eq.~\eqref{eq:condition2} is
  \begin{align*}
   &R_1  = q_0\sum_{j=0}^{M} f^{\prime}(j)\times p_0\times \dots \times p_{j-1}\times L(\{q_0,\dots, q_j\}, t-j) \tag{$(1 - s_{i^*})F_{g}(\bm{s}, t)$}\\
       & +p_0 \sum_{j=0}^{M^{\prime}}f^{\prime}(j + 1)\times p_1\times\dots \times p_{j}\times L(\{q_1,\dots, q_{j+1}\}, t - j).\tag{$s_{i^*}F_{g}(\bm{s} - \mu_{i^*} \bI_{i^*}, t)  + s_{i^*}\Delta_c$}
  \end{align*}
  In fact, $R_1 = F_g(\bm{s}, t + 1)$ based on Lemma~\ref{lemma:property_greedy_policy}.
  The first line of Eq.~\eqref{eq:condition2} is  
  \begin{align*}
   R_2 & = q_k\sum_{j=0}^Mf^{\prime}(j)\times p_0\times\dots\times p_{j-1}\times L(\{q_0,\dots, q_j\}, t - j) \\
       & + p_k \sum_{j=0}^{k-1}f^{\prime}(j + 1)\times p_0\times \dots\times p_{j-1} \times L(\{q_0, \dots, q_j\}, t - j) \\
       & + \sum_{j=k}^{M^{\prime}}f^{\prime}(j + 1)\times p_0\times \dots\times p_j \times L(\{q_0, \dots, q_{j+1}\}\backslash\{q_k\}, t - j).
  \end{align*}
  Our goal is to measure the gap $R_1 - R_2$. Let $\textrm{Prob}_{\bm{s}, t}(i)$ be
  the probability we can meet $i$ distinct members if we explore communities
  (whose initial status is $\bm{s}$) with
  greedy policy for $t$ steps. According to
  Lemma~\ref{lemma:small_gap}, we have
  \begin{align*}
   F_g(\bm{s}, t + 1) - A(\bm{s}, i, t) &= \sum_{j=0}^{k-1} \left(R(j) - R(j + 1)\right)\\
& = \sum_{j=0}^{k-1} (p_j - p_{j+1}) \left(\sum_{o=0}^{j}(f^{\prime}(o + 1) - f^{\prime}(o))\textrm{Prob}_{\bm{s},t}(o)\right)\\
    & = \sum_{o=0}^{k-1}(f^{\prime}(o + 1) - f^{\prime}(o))\textrm{Prob}_{\bm{s},t}(o)\left( \sum_{j=o}^{k-1} p_j - p_{j+1}\right)\\
    & = \sum_{j=0}^{k-1}(f^{\prime}(j + 1) - f^{\prime}(j))(p_j - p_k)\textrm{Prob}_{\bm{s},t}(j)
  \end{align*}
  When the reward equals to the number of distinct members, we have
  \begin{equation*}
    F_g(\bm{s}, t + 1) - A(\bm{s}, i, t) = \sum_{j=0}^{k-1}(p_j - p_k)\textrm{Prob}_{\bm{s},t}(j).
  \end{equation*}
  Besides, the gap $F_g(\bm{s}, t + 1) - A(\bm{s}, i, t)$ increases as $k$
  increases, which means the worse choice we have at first, the larger
  reward gap we have at end.

\section{Basics of online learning problems}
\subsection{Set size estimation by collision counting}
Suppose we have a set $C_i = \{u_1, \cdots, u_{d_i}\}$ whose population $d_i$ is
unknown. Let $u, v$ be two elements
selected with replacement from $C_i$, and $Y_{u,v}$ denote a random
variable that takes value 1 if $u = v$ ({\em a collision}) and $0$ otherwise.
The expectation of $Y_{u, v}$ equals to $\frac{1}{d_i}$, i.e., $\mathbb{E}[Y_{u,
v}] = \frac{1}{d_i}$. 
Assume we sample $k_i$ elements {\em with replacement} uniformly at random
from set $C_i$. Let $\cS_i$ be the set of samples. 
With the sample $\cS_i$, we compute the estimator for $d_i$ as
\begin{equation*}
 \hat{d_i} = \frac{k_i(k_i - 1)}{2X_i}, 
\end{equation*}
here $X_i = \sum_{u\in \cS_{i}, v\in \cS_i\backslash\{u\}} Y_{u, v}$ is the
number of collisions in $\cS_i$. According to the Jensen's
inequality\footnote{If $X$ is a random variable, and $\varphi$ is a convex
function, then $\varphi(\mathbb{E}[X])\leq \mathbb{E}[\varphi(X)]$.}, we have
$d_i\leq \mathbb{E}[\hat{d}_i]$, i.e., $\hat{d}_i$ is a biased estimator. The
estimator is invalid when $X_i = 0$. Since the equality only occurs when
$\text{Var}[X_i] = 0$, which is not the case Here. We have $d_i <
\mathbb{E}[\hat{d}_i]$.

\textbf{Independence.} Let $\cS_i = \{v_1,\cdots, v_{k_i}\}$. For the two random
variable $Y_{v_x, v_y}$ ($1\leq x < y\leq k_i$) and $Y_{v_{x^{\prime}},
v_{y^{\prime}}}$ ($1\leq x^{\prime} < y^{\prime}\leq k_i$), we consider three
difference cases.
\begin{enumerate}[leftmargin=*]
\item There are ${k_i \choose 2}$ occurrences when $x = x^{\prime}, y =
  y^{\prime}$. Here $\mathbb{E}[Y_{v_x, v_y}Y_{v_{x^{\prime}}, v_{y^{\prime}}}] = 1/d_i$.
\item There are $6{k_i \choose 3}$ occurrences when $x = x^{\prime}, y \neq
  y^{\prime}$ or $x\neq x^{\prime}, y = y^{\prime}$. $\mathbb{E}[Y_{v_x, v_y}Y_{v_{x^{\prime}}, v_{y^{\prime}}}] = 1/d^2_i$.
\item There are $6{k_i \choose 4}$ occurrences when $x \neq x^{\prime}, y \neq
  y^{\prime}$. Here $\mathbb{E}[Y_{v_x, v_y}Y_{v_{x^{\prime}}, v_{y^{\prime}}}] = 1/d^2_i$.
\end{enumerate}
We say that pairs $(v_x, v_y)$ and $(v_{x^{\prime}}, v_{y^{\prime}})$ are
different if $x\neq x^{\prime}$ or $y \neq y^{\prime}$. When $(v_x, v_y)$
and $(v_{x^{\prime}}, v_{y^{\prime}})$ are different, we have
$\mathbb{E}[Y_{v_x, v_y}Y_{v_{x^{\prime}}, v_{y^{\prime}}}] = \mathbb{E}[Y_{v_x,
v_y}]\mathbb{E}[Y_{v_{x^{\prime}}, v_{y^{\prime}}}] = 1/d^2_i$. Above discussion
indicates that the ${k_i\choose 2}$ pairs of random variables obtained from
$\cS_i$ are {\em $2$-wise independent}.

\textbf{Variance.} We compute the variance $\text{Var}[X_i] = \mathbb{E}[X^2_i]
- \mathbb{E}^2[X_i]$ in the following.
\begin{equation*}
  \begin{split}
   \text{Var}[X_i] &  = \frac{k_i(k_i - 1)}{2d_i} + \frac{k_i(k_i-1)(k_i-2)}{d_i^2} + \frac{k_i(k_i-1)(k_i-2)(k_i-3)}{4d_i^2} - \frac{k_i^2(k_i-1)^2}{4d_i^2}\\
   & =  {k_i\choose 2}\frac{1}{d_i}(1 - \frac{1}{d_i}) = {k_i \choose 2}\text{Var}[Y_{u, v}].
  \end{split}
\end{equation*}

\noindent\textbf{Collision} Since the estimator is based on the collision
counting, we need to ensure that $X_i > 0$ with high probability. Let $B_{k_i}$ denote
the event that the $k_i$ samples $\{v_{1}, \dots, v_{k_i}\}$ are distinct. We
have
\begin{equation*}
  \begin{split}
 \Pr\{B_k\} = 1\cdot (1 - \frac{1}{d_i})(1 - \frac{2}{d_i}) \cdots (1 - \frac{k_i - 1}{d_i})&\leq e^{-1/d_i}e^{-2/d_i}\cdots e^{-(k_i-1)d_i}\\
&= e^{-\sum_{j=1}^{k_i-1}j/d_i} =  e^{-k_i(k_i-1)/2d_i}.
  \end{split}
\end{equation*}
To ensure that $X_i > 0$ with probability no less than $1 - \delta$, we have 
\begin{equation*}
 k_i \geq \left(1 + \sqrt{8d_i\ln \frac{1}{\delta} + 1}\right)  / 2.
\end{equation*}

\subsection{Concentration bound for variables with local dependence}
Note that the pairs $Y_{u, v}$ and $Y_{u^{\prime}, v^{\prime}}$ are not mutually
independent. Actually, their dependence can be described with a {\em dependence
graph}~\cite{janson2004large, Dubhashi2009CMA}. The Chernoff-Hoeffding bound
in~\cite{hoeffding1963probability} can not be used directly for our estimator of
$\mu_i$. In the following, we present a concentration bound that is applicable
to our problem.

\begin{definition}[U-statistics]
  Let $\xi_1, \dots, \xi_n$ be independent random variables, and let
  \begin{equation*}
   X \defeq \sum_{1\leq i_1\leq \dots\leq i_d} f_{i_1,\dots, i_d} (\xi_{i_1}, \dots, \xi_{i_d}).
  \end{equation*}
\end{definition}

\begin{lemma}[Chapter 3.2~\cite{Dubhashi2009CMA}]\label{lemma:local_dependence}
 If $a\leq f_{i_1, \dots, i_d}(\xi_{i_1}, \dots, \xi_{i_d})\leq b$ for every
 $i_1, \dots, i_d$ for some reals $a\leq b$, we have
 \begin{equation*}
  \Pr\left\{ |X -  \mathbb{E}[X]| \geq \epsilon{n\choose d}\right\} \leq 2\exp\left( \frac{-2\floor{n/d}\epsilon^2}{(b-a)^2} \right).
 \end{equation*}
\end{lemma}

In our problem, if we get $k_i$ samples from set $C_i$, then the number of
collisions satisfies 
\begin{equation*}
 \Pr\left\{\abs{X_i - \mathbb{E}[X_i]} \geq \epsilon {k_i\choose 2}\right\} \leq 2\exp\left( -2\floor{k_i/2}\epsilon^2 \right).
\end{equation*}
Above inequality indicates that the actual number of independent pairs is
$\floor{k_i/2}$ when using collisions in $k_i$ samples to estimate $\mu_i$. 

\section{Regret Analysis for Non-Adaptive Problem}
\subsection{Supporting Corollaries}

\begin{corollary}\label{corollary:action_summation}
 For action $\bm{k}$ with $\sum_{i=1}^{m}k_i = K$ and $k_i \geq 1$, we have
 $\sum_{i=1}^m{k_i\choose 2} \leq {K - m + 1\choose 2}$.  
\end{corollary}
\begin{proof}
  We prove the corollary by simple calculation.
  \begin{align*}
   \sum_{i=1}^m{k_i\choose 2} - {K - m + 1\choose 2} &= \frac{1}{2}\left( \sum_{i=1}^mk_i(k_i - 1) - \left(1 + \sum_{i=1}^m(k_i - 1)\right)\left(\sum_{i=1}^m(k_i - 1)\right)\right) \\
   &= \frac{1}{2}\left( \sum_{i=1}^m(k_i-1)^2 - \left(\sum_{i=1}^m(k_i - 1)\right)^2\right) \leq 0. \qedhere
  \end{align*}
\end{proof}

\subsection{Basics}

To compare with the CUCB algorithm introduced in~\cite{wang2017improving} for
general CMAB problem, we propose an revised Algo.~\ref{algo:CLCB_algorithm} that
is consistent with the CUCB algorithm in~\cite{wang2017improving}. We revise the
Line~\ref{line:update_pairs}-\ref{line:online_learning_end} in
Algo.~\ref{algo:CLCB_algorithm} as follows.
\begin{equation}\label{eq:revision_1}
  \begin{split}
& \text{Line~\ref{line:update_pairs}:~~ For } i\in [m], T_i\leftarrow T_i + \bOne\{\abs{\cS_i} > 1\},\\ 
& \text{Line~\ref{line:count_collision}:~~ For } i\in [m] \text{ and } {\abs{\cS_i} > 1}, X_{i,t} \leftarrow \sum_{x=1}^{\floor{\abs{\cS_i}} / 2}\bOne\{u_{2x - 1} = u_{2x}\} / \floor{\abs{\cS_i/2}},\\
& \text{Line~\ref{line:online_learning_end}: For } i\in [m] \text{ and } {\abs{\cS_i} > 1}, \hat{\mu}_i\leftarrow \hat{\mu}_i + (X_{i,t} - \hat{\mu}_i) / T_i.
  \end{split}
\end{equation}
Note that $\hat{\mu}_i$ in Eq.~\eqref{eq:revision_1} is also an unbiased
estimator of $\mu_i$. Then we can obtain the regret bound of the revised
Algo.~\ref{algo:CLCB_algorithm} by applying the Theorem 4 in the extended
version of~\cite{wang2017improving} directly.
\begin{equation*}
 \text{Reg}_{\bm{\mu}}(T) \leq \sum_{i=1}^{m} \frac{48{K - m + 1 \choose 2}^2 m\ln T }{\Delta^{i}_{\min}} + 2{K - m + 1 \choose 2}m + \frac{\pi^2}{3}\cdot m \cdot \Delta_{\max}. 
\end{equation*}

We add superscript $r$ to differentiate the corresponding random
variables in the revised Algo.~\ref{algo:CLCB_algorithm} from the original ones.
E.g., $T^{r}_{i, t}$ is the value of $T^{r}_i$ in the revised
Algo.~\ref{algo:CLCB_algorithm} at the end of round $t$. Recall that $K^{\prime} =
K - m + 1$, which is the maximum exploration times for a community in each
round.

\subsection{Proof framework}

We first introduce a definition which describes the event that
$\hat{\mu}_{i,t-1}$ ($\hat{\mu}^{r}_{i, t-1}$) is accurate at the beginning of
round $t$.
\begin{definition}
  We say that the sampling is nice at the beginning of round $t$ if for every
  community $i\in [m]$, $|\hat{\mu}_{i,t-1} - \mu_i| \leq \rho_{i, t}$ (resp.
  $|\hat{\mu}^{r}_{i,t-1} - \mu_i| \leq \rho^{r}_{i, t}$), where $\rho_{i, t} =
  2\sqrt{\frac{3\ln t}{2T_{i,t-1}}}$ (resp. $\rho^{r}_{i, t} = 2\sqrt{\frac{3\ln
      t}{2T^{r}_{i,t-1}}}$) in round $t$. Let $\mathcal{N}_t$ (resp.
  $\mathcal{N}^{r}_t$) be such event.
\end{definition}

\begin{lemma}\label{lemma:small_probability_for_event}
 For each round $t \geq 1$, $\Pr\left\{ \neg \mathcal{N}_t\right\}\leq
 2m\floor{K^{\prime}/2}t^{-2}$ (resp. $\Pr\left\{ \neg \mathcal{N}^{r}_t\right\}\leq 2mt^{-2}$).
\end{lemma}
\begin{proof}
  For each round $t \geq 1$, we have
  \begin{align*}
   \Pr\left\{\neg \mathcal{N}_t \right\} &= \Pr\left\{\exists i\in [m], |\hat{\mu}_{i, t-1} - \mu_i| \geq \sqrt{\frac{3\ln t}{2T_{i, t-1}}}\right\}\\ 
  & \leq  \sum_{i\in [m]} \Pr \left\{|\hat{\mu}_{i, t-1} - \mu_i| \geq \sqrt{\frac{3\ln t}{2K_{i, t-1}}} \right\}\\
  & = \sum_{i\in [m]}\sum_{k=1}^{(t-1)\floor{K^{\prime}/ 2}} \Pr\left\{T_{i, t-1} = k, |\hat{\mu}_{i, t-1} - \mu_i| \geq  \sqrt{\frac{3\ln t}{2T_{i, t-1}}}\right\}\\
  &\leq  \sum_{i\in [m]} \sum_{k=1}^{(t-1)\floor{K^{\prime}/ 2}} \frac{2}{t^3} < 2{m\floor{K^{\prime} /2}}t^{-2}.\tag{Hoeffding's inequality~\cite{hoeffding1963probability}}
  \end{align*}
  When $T_{i,t - 1} = k$, $\hat{\mu}_{i,t}$ is the average of $k$ i.i.d. random
  variables $Y^{[1]}_i, \cdots, Y^{[k]}_i$, where $Y^{[j]}_i$ is a random variable
  that indicates whether two members selected with replacement from $C_i$ are the
  same. Since each community is explored at most $K^{\prime}$ times in each round,
  $T_{i,t-1}\leq (t-1)\floor{K^{\prime}/2}$. The last line leverages the
  Hoeffding's inequality~\cite{hoeffding1963probability}. By replacing the
  summation range $k\in [1, (t-1)\floor{K^{\prime}/ 2}]$ with $k\in [1, (t-1)]$ in
  the 3rd line of above equation, we have $\Pr\left\{ \neg
    \mathcal{N}^{r}_t\right\}\leq 2mt^{-2}$.
\end{proof}

Secondly, we use the monotonicity and bounded smoothness properties to bound the
reward gap $\Delta_{\bm{k}_t} = r_{\bm{k}^*}(\bm{\mu}) - r_{\bm{k}_t}(\bm{\mu})$
between our action $\bm{k}_t$ and the optimal action $\bm{k}^*$.

\begin{lemma}\label{lemma:bound_gap}
 If  the event $\mathcal{N}_t$ holds in round $t$, we have
\begin{equation*}
 \Delta_{\bm{k}_t} \leq\sum_{i=1}^m \binom{k_{i,t}}{2}\kappa_{T}(\Delta^i_{\min}, T_{i,t-1}).
\end{equation*}
Here the function $\kappa_T(M, s)$ is defined as 
\begin{equation*}
 \kappa_T(M, s) =
 \begin{cases}
   2 & \text{ if } s = 0, \\
   2\sqrt{\frac{6\ln t}{s}} & \text{ if } 1 \leq s \leq l_{T}(M),\\ 
   0 & \text{ if } s \geq l_{T}(M) + 1,
 \end{cases}
\end{equation*}
where 
\begin{equation*}
 l_{T}(M) = \frac{24{K^{\prime}\choose 2}^2\ln T}{M^2}. 
\end{equation*}
\end{lemma}
\begin{proof}
  By $\mathcal{N}_t$ (i.e.,  $\ubar{\bm{\mu}}_t \leq \bm{\mu}$) and
  the monotonicity of $r_{\bm{k}}(\bm{\mu})$, we have
\begin{equation*}
 r_{\bm{k}_t}(\ubar{\bm{\mu}}_t) \geq r_{\bm{k}^*}(\ubar{\bm{\mu}}_t) \geq r_{\bm{k}^*}(\bm{\mu}) = r_{\bm{k}_t}(\bm{\mu}) + \Delta_{\bm{k}_t}.
\end{equation*}
Then by the \textit{bounded smoothness} properties of reward function, we have
\begin{equation*}
  \Delta_{\bm{k}_t} \leq r_{\bm{k}_t}(\ubar{\bm{\mu}}_t) - r_{\bm{k}_t}(\bm{\mu}) \leq \sum_{i=1}^{m}{k_{i,t}\choose 2}(\mu_i - \ubar{\mu}_{i, t}). 
\end{equation*}
We intend to bound $\Delta_{\bm{k}_t}$ by bounding ${\mu}_i - \ubar{\mu}_{i,
t}$. Before doing so, we perform a transformation. Let $M_{\bm{k}_t} =
\max_{i\in [m], k_{i,t} > 1}\Delta^{i}_{\min}$. Since the action $\bm{k}_t$
always satisfies $\Delta_{\bm{k}_t} \geq \max_{i\in [m], k_{i,t} >
1}\Delta^{i}_{\min}$, we have $\Delta_{\bm{k}_t} \geq M_{\bm{k}_t}$. So
$\sum_{i} {k_{i,t}\choose 2}(\mu_i - \ubar{\mu}_{i, t})\geq \Delta_{\bm{k}_t}
\geq M_{\bm{k}_t}$. Therefore,
\begin{align*}
 \Delta_{\bm{k}_t} &\leq \sum_{i=1}^{m} {k_{i,t}\choose 2}(\mu_i - \ubar{\mu}_{i,t}) \leq - M_{\bm{k}_t} + 2\sum_{i = 1}^{m} {k_{i,t}\choose 2}(\mu_{i}  - \ubar{\mu}_{i, t})\\
 & \leq - \frac{\sum_{i=1}^{m} {k_{i,t} \choose 2}}{{K^{\prime}\choose 2}}M_{\bm{k}_t} + 2\sum_{i = 1}^{m} {k_{i,t}\choose 2}(\mu_{i}  - \ubar{\mu}_{i, t})\tag{Corollary~\ref{corollary:action_summation}: $\sum_{i=1}^m{k_{i,t}\choose 2}\leq {K^{\prime}\choose 2}$}\\
 & = 2\sum_{i=1}^{m}{k_{i,t}\choose 2}\left[(\mu_i - \ubar{\mu}_{i, t}) - \frac{M_{\bm{k}_t}}{K^{\prime}(K^{\prime}-1)}\right]\\
 & \leq 2\sum_{i}{k_{i,t}\choose 2}\left[(\mu_i - \ubar{\mu}_{i, t}) - \frac{\Delta^{i}_{\min}}{K^{\prime}(K^{\prime}-1)}\right].\tag{by definition of $M_{\bm{k}_t}$}
\end{align*}
By $\mathcal{N}_t$, we have $\mu_i - \ubar{\mu}_{i, t}\leq \min\{2\rho_{i,t}, 1\}$. So
\begin{equation*}
  \begin{split}
    \mu_i - \ubar{\mu}_{i, t} - \frac{\Delta^i_{\min}}{K^{\prime}(K^{\prime}-1)} &\leq \min\{2\rho_{i,t}, 1\} - \frac{\Delta^i_{\min}}{K^{\prime}(K^{\prime}-1)} \leq \min\left\{ \sqrt{\frac{6\ln t}{T_{i,t-1}}} , 1\right\} - \frac{\Delta^i_{\min}}{K^{\prime}(K^{\prime}-1)}.
  \end{split}
\end{equation*}
If $T_{i, t-1}\leq l_T(\Delta^i_{\min})$, we have $\mu_i - \ubar{\mu}_{i, t}  - \frac{\Delta^i_{\min}}{K^{\prime}(K^{\prime}-1)}
\leq \min\left\{\sqrt{\frac{6\ln t}{T_{i,t-1}}} , 1\right\}\leq \frac{1}{2}\kappa_{T}(\Delta^i_{\min}, T_{i, t-1})$. If $T_{i, t-1} > l_T(\Delta^i_{\min}) + 1$,
then $ \sqrt{\frac{6\ln t}{T_{i,t-1}}} \leq \frac{\Delta^i_{\min}}{K^{\prime}(K^{\prime}-1)}$, so $(\mu_i - \ubar{\mu}_{i,t} ) - \frac{\Delta^i_{\min}}{K^{\prime}(K^{\prime}-1)} \leq 0 =
\kappa_T(\Delta^i_{\min}, T_{i, t-1})$. In conclusion, we have
\begin{equation*}
  \Delta_{\bm{k}_t} \leq \sum_{i=1}^{m}{k_{i,t}\choose 2}\kappa_{T}(\Delta_{\min}^i, T_{i, t-1}). \qedhere
\end{equation*}
\end{proof}

Above result is also valid for the revised Algo.~\ref{algo:CLCB_algorithm},
i.e., $\Delta^{r}_{\bm{k}_t} \leq \sum_{i=1}^{m}{k_{i,t}\choose
2}\kappa_{T}(\Delta_{\min}^i, T^{r}_{i, t-1})$. Our third step is to prove that
when $\mathcal{N}_{t}$ (resp. $\mathcal{N}^{r}_t$) holds, the regret is bounded
in $O(\ln T)$.

\begin{figure}[t]
  \centering
 \includegraphics[scale=0.8]{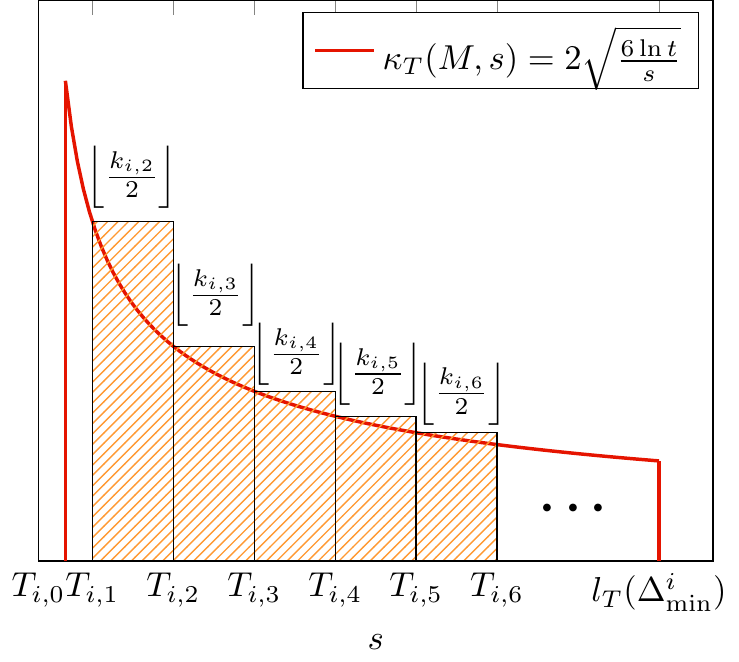} 
 \caption{Demonstration of the regret summation $\sum_{t=2}^T\floor{k_{i,t}/
     2}\kappa_T(\Delta^{i}_{\min}, T_{i,t-1})$. It is obvious that when
   $k_{i,t} = K^{\prime}$, then the shaded area (colored with orange) covered by
   the rectangles is maximized.}\label{fig:regret_riemann_sum}
\end{figure}

\regretbounda*
\begin{proof}
  We first prove the regret when the event $\mathcal{N}_t$ holds. In each run,
  we have
  \begin{align*}
  \sum_{t=1}^{T} \bm{1}(\left\{\Delta_{\bm{k}_t}  \wedge \mathcal{N}_t\right\})\cdot \Delta_{\bm{k}_t} &\leq \sum_{t=1}^T \sum_{i=1}^{m} {k_{i,t}\choose 2}\kappa_T(\Delta^i_{\min}, T_{i, t-1})\\
&=  \sum_{i=1}^{m}\sum_{t^{\prime}\in \{t \mid 1\leq t\leq T, k_{i,t} > 1\}} {k_{i,t^{\prime}}\choose 2}\kappa_T(\Delta^i_{\min}, T_{i, t^{\prime}-1}).
    \end{align*}
    Hence, we just assume $k_{i, t} > 1$ for $t > 0$.
  \begin{align*}
    \small
\sum_{t=1}^{T} \bm{1}(\left\{\Delta_{\bm{k}_t}  \wedge \mathcal{N}_t\right\})\cdot \Delta_{\bm{k}_t} &\leq \sum_{i=1}^{m}\sum_{t=1}^{T} {k_{i,t}\choose 2}\kappa_T(\Delta^{i}_{\min}, T_{i, t-1})\\
    &\leq \sum_{i=1}^{m}2{k_{i,1}\choose 2} + K^{\prime}\sum_{i=1}^{m}\sum_{t=2}^{T} \frac{(k_{i,t}-1)}{2}\kappa_T(\Delta^{i}_{\min}, T_{i, t-1})\\
    &\leq 2m{K^{\prime}\choose 2} + K^{\prime}\sum_{i=1}^{m}\sum_{t=2}^{T} \floor*{\frac{k_{i,t}}{2}}\kappa_T(\Delta^{i}_{\min}, T_{i, t-1}).\tag{Fig.~\ref{fig:regret_riemann_sum}}
  \end{align*}
  To maximize the summation $\sum_{t=2}^{T} \floor{\frac{k_{i,t}}{2}}\kappa_T(\Delta^{i}_{\min}, T_{i, t-1})$, we just need to let $k_{i,t} =
K^{\prime}$ when $t > 1$. 
  \begin{align*}
    \small
\sum_{t=1}^{T} \bm{1}(\left\{\Delta_{\bm{k}_t}  \wedge \mathcal{N}_t\right\})\cdot \Delta_{\bm{k}_t}  &\leq 2m{K^{\prime}\choose 2} + K^{\prime}\sum_{d=0}^{l_{T}(\Delta^{i}_{\min})/\floor{K^{\prime}/ 2}}\floor*{\frac{K^{\prime}}{2}}\kappa_T\left(\Delta^{i}_{\min}, 1 + d\floor{K^{\prime}/2}\right) \\
    & \leq 2m{K^{\prime}\choose 2} + K^{\prime}\sum_{i=1}^{m}\sum_{d=0}^{l_{T, K}} \frac{\sqrt{24\ln T} \floor{K^{\prime}/ 2}}{\sqrt{1 + d\floor{K^{\prime}/ 2}}} \tag{$l_{T, K}\defeq \frac{l_{T}(\Delta^{i}_{\min})}{\floor{K^{\prime} / 2}}$}\\
    & \leq 2m{K^{\prime}\choose 2} + K^{\prime}\sum_{i=1}^{m} \int_{x=0}^{l_{T, K}} \frac{\sqrt{24\floor{K^{\prime}/ 2}\ln T}}{\sqrt{x}}dx \\
    & = 2m{K^{\prime}\choose 2} +  K^{\prime}\sum_{i=1}^{m}\sqrt{96l_T(\Delta^{i}_{\min}, T)\ln T} \\ 
    & = 2m{K^{\prime}\choose 2} +  \sum_{i=1}^{m}\frac{48{K^{\prime}\choose 2}K^{\prime}\ln T}{\Delta^{i}_{\min}}. 
\end{align*}
On the other hand, when $\mathcal{N}_t$ does not hold, we can bound the regret
as $\Delta_{\max}$. Hence,
\begin{align*}
  \mathbb{E} \left[ \sum_{t=1}^{T} \bm{1}(\left\{\Delta_{\bm{k}_t}  \wedge \neg\mathcal{N}_t\right\})\cdot \Delta_{\bm{k}_t} \right] \leq\Delta_{\max}\sum^{T}_{t=1}2{m\floor{K^{\prime} / 2}}t^{-2}\leq \frac{m\floor{K^{\prime}/ 2}\pi^2}{3}\Delta_{\max}.
\end{align*}

Based on above discussion, we have
\begin{align*}
    \text{Reg}_{\bm{\mu}}(T) &\leq  \frac{m\floor{K^{\prime} / 2}\pi^2}{3}\Delta_{\max} + 2m{K^{\prime}\choose 2} +  \sum_{i=1}^{m}\frac{48{K^{\prime}\choose 2}K^{\prime}\ln T}{\Delta^{i}_{\min}}. \qedhere
\end{align*}
\end{proof}

\begin{restatable}{theorem}{regretboundb}\label{thm:regret_bound_non_adaptive_exploration_2}
The revised Algo.~\ref{algo:CLCB_algorithm} has regret as follows.
 \begin{align}\label{eq:non_adaptive_regret_bound_b}
    \text{Reg}^{r}_{\bm{\mu}}(T) &\leq  \sum_{i=1}^{m}\frac{48{K^{\prime} \choose 2}^2\ln T}{\Delta^{i}_{\min}} + 2{K^{\prime}\choose 2}m + \frac{\pi^2}{3} \cdot m \cdot \Delta_{\max}.
\end{align}
\end{restatable}
\begin{proof}
 We prove the regret when the event $\mathcal{N}^r_t$ holds. In each run, we
 have
   \begin{align*}
  \sum_{t=1}^{T} \bm{1}(\left\{\Delta^r_{\bm{k}_t}  \wedge \mathcal{N}^r_t\right\})\cdot \Delta^r_{\bm{k}_t} &\leq \sum_{t=1}^T \sum_{i=1}^{m} {k_{i,t}\choose 2}\kappa_T(\Delta^i_{\min}, T^r_{i, t-1})\\
&=  \sum_{i=1}^{m}\sum_{s = 0}^{T^r_{i, T}} {k_{i,s}\choose 2}\kappa_T(\Delta^i_{\min}, s)\\
& \leq 2m{K^{\prime}\choose 2} + {K^{\prime}\choose 2}\sum_{i=1}^{m}\sum_{s = 1}^{l_{T}(\Delta^i_{\min})} \sqrt{\frac{24\ln T}{s}}\\
& \leq 2m{K^{\prime}\choose 2} +  \sum_{i=1}^{m}\frac{48{K^{\prime}\choose 2}^2\ln T}{\Delta^{i}_{\min}}.
    \end{align*}
    On the other hand, $\Pr\{\neg \mathcal{N}^r_t\} \leq 2mt^{-2}$. Hence we
    have
    \begin{align*}
     \text{Reg}^r_{\bm{\mu}}(T) &=  \mathbb{E} \left[ \sum_{t=1}^{T} \bm{1}(\left\{\Delta_{\bm{k}_t}  \wedge \neg\mathcal{N}_t\right\})\cdot \Delta_{\bm{k}_t} \right] + \mathbb{E} \left[ \sum_{t=1}^{T} \bm{1}(\left\{\Delta_{\bm{k}_t}  \wedge \mathcal{N}_t\right\})\cdot \Delta_{\bm{k}_t} \right]\\
    &  \leq  \frac{m\pi^2}{3}\Delta_{\max} + 2m{K^{\prime}\choose 2} +  \sum_{i=1}^{m}\frac{48{K^{\prime}\choose 2}^2\ln T}{\Delta^{i}_{\min}}.\qedhere
    \end{align*}
\end{proof}
The bound in Eq.~\eqref{eq:non_adaptive_regret_bound_b} is tighter than the one
obtained by directly applying~\cite{wang2017improving}.

\subsection{Comparison}
\textbf{Estimator.} Let $\hat{\mu}_{i,t}$ be the estimator computed in
Algo.~\ref{algo:CLCB_algorithm} by end of round $t$ and $\hat{\mu}^{r}_{i,t}$ be
the estimator computed with revision in Eq.~\eqref{eq:revision_1} by end of
round $t$. Both of $\hat{\mu}_{i,t}$ and $\hat{\mu}^{r}_{i,t}$ are unbiased
estimator of $\mu_i$. However, $\hat{\mu}_{i,t}$ is a {\em more efficient}
estimator than $\hat{\mu}^{r}_{i,t}$. More specifically, $\Var[\hat{\mu}_{i,t}]
= \mu_i(1 - \mu_i) / (\sum_{t^{\prime} = 1}^{t}\floor{k_{i,t^{\prime}}/ 2})$ and
$\Var[\hat{\mu}^{r}_{i,t}] = \mu_i(1 - \mu_i)\cdot (\sum_{t^{\prime} = 1}^t1 /
\floor{k_{i,t}/ 2}) / (T^r_{i,t})^2$. Here $k_{i,t}$ is the size of $\cS_i$ in
round $t$, and $T^{r}_{i,t} = \sum_{t^{\prime}=1}^t\bOne\{k_{i,t^{\prime}} >
1\}$. Since the harmonic mean is always not larger than arithmetic mean, i.e., $
T^{r}_{i,t} / (\sum_{t^{\prime} = 1}^t1 / \floor{k_{i,t^{\prime}}/ 2}) \leq
(\sum_{t^{\prime} = 1}^{t}\floor{k_{i,t^{\prime}}/2}) / T^{r}_{i,t}$, we
conclude that $\Var[\hat{\mu}_{i,t}]\leq \Var[\hat{\mu}^{r}_{i,t}]$.

\textbf{Regret Bound.} The regret bound in
Eq.~\eqref{eq:non_adaptive_regret_bound_a} is tighter than the one in
Eq.~\eqref{eq:non_adaptive_regret_bound_b} up to ${(K^{\prime} - 1)/ 2}$ factor
in the $O(\ln T)$ term. The bound in Eq.~\eqref{eq:non_adaptive_regret_bound_a}
has a larger constant term. That's because we use a smaller confidence radius,
which leads to earlier exploitation of Algo.~\ref{algo:CLCB_algorithm} than the
revised one.

\subsection{Full information feedback}
In the following, we prove the constant regret bound of the
Algo.~\ref{algo:CLCB_algorithm} with feeding the empirical mean in
\texttt{\textproc{CommunityExplore}} and making revision defined in
Eq.~\eqref{eq:revision_2}.
\begin{proof}
  We first bound $\Delta_{\bm{k}_t}$ by $\sum_{i=1}^m \abs{\mu_{i,t} - \mu_i}$.
\begin{align*}
 \Delta_{\bm{k}_t} = r_{\bm{k}^{*}}(\bm{\mu}) - r_{\bm{k}_t}(\bm{\mu})  &=  r_{\bm{k}^{*}}(\bm{\mu}) - r_{\bm{k}_t}(\hat{\bm{\mu}}) +  r_{\bm{k}_t}(\hat{\bm{\mu}})- r_{\bm{k}_t}(\bm{\mu})\\
 & \leq  r_{\bm{k}^{*}}(\bm{\mu}) - r_{\bm{k}^{*}}(\hat{\bm{\mu}}) +  r_{\bm{k}_t}(\hat{\bm{\mu}})- r_{\bm{k}_t}(\bm{\mu})\tag{$r_{\bm{k}^{*}}(\hat{\bm{\mu}}) \leq r_{\bm{k}_t}(\hat{\bm{\mu}})$}\\
 & \leq  |r_{\bm{k}^{*}}(\bm{\mu}) - r_{\bm{k}^{*}}(\hat{\bm{\mu}})| +  |r_{\bm{k}_t}(\hat{\bm{\mu}})- r_{\bm{k}_t}(\bm{\mu})|\\
 &\leq \sum_{i=1}^{m}\left( {k^{*}_i\choose 2} + {k_{i,t}\choose 2}\right)\abs{\hat{\mu}_{i,t-1} - \mu_i}.
\end{align*}
  Leverage the fact that $\sum_{i=1}^m {k_{i, t}\choose 2} \leq
{K^{\prime}\choose 2}$. If $|\hat{\mu}_{i,t-1} - \mu_i| <
\frac{\Delta_{\min}}{K^{\prime}(K^{\prime} - 1)}$, then
\begin{align*}
  \Delta_{\bm{k}_t} \leq \sum_{i=1}^{m}\left( {k^{*}_i\choose 2} + {k_{i,t}\choose 2}\right) \frac{\Delta_{\min}}{K^{\prime}(K^{\prime} - 1)} < \Delta_{\min},
\end{align*}
which means $\Delta_{\bm{k}_t} = 0$. Hence, 
\begin{align*}
 \Pr\left(\Delta_{\bm{k}_t} > 0 \right) &\leq \sum_{i=1}^{m}\Pr\left( |\hat{\mu}_{i,t-1} - \mu_i| \geq \frac{\Delta_{\min}}{K^{\prime}(K^{\prime} - 1)}\right)\\
 & \leq \sum_{i=1}^{m}2e^{-2(T_{i,t-1} / 2)\Delta^2_{\min} / (K^{\prime}(K^{\prime}-1))^2} \tag{Theorem 3.2 in~\cite{Dubhashi2009CMA}}. 
\end{align*}
The second line of above inequality using Theorem 3.2 in~\cite{Dubhashi2009CMA}.
Note that the $T_{i, t-1}$ member pairs using for collision counting are not
independent with each other. We need to construct a {\em dependence graph} $G$
to model their dependence. The dependence graph here is just a line with $T_{i,
t-1}$ nodes. Since the fractional chromatic number of the dependence graph is
$2$, we have a $1/2$ factor for $T_{i,t-1}$ in the exponential. The regret is
bounded as
\begin{align*}
\text{Reg}_{\bm{\mu}}(T)  &\leq \sum_{t=1}^{T}\sum_{i=1}^{m}\Delta_{\bm{k}_t} 2e^{-T_{i,t-1}\Delta^2_{\min} / (K^{\prime}(K^{\prime}-1))^2}\\
&\leq 2\Delta_{\max} + \sum_{i=1}^{m}\sum_{t=3}^{T}\Delta_{\bm{k}_t} 2e^{-(t-2)\Delta^2_{\min} / (K^{\prime}(K^{\prime}-1))^2}\tag{$T_{i,t-1} \geq t - 2$}\\
& \leq 2\Delta_{\max} + 2m\Delta_{\max} \int_{t=0}^{\infty}e^{-t\Delta^2_{\min} / (K^{\prime}(K^{\prime}-1))^2}{\rm d}t\\
& \leq \left(2 + 8me^2{K^{\prime}\choose 2}^2 / \Delta^2_{\min}\right)\Delta_{\max}.\qedhere
\end{align*}
\end{proof}

\section{Regret Analysis for Adaptive Problem}
\subsection{Transition probability list of policy $\pi^t$}\label{app:transition_probability_list_pit}
Similar to the discussion in
Section~\ref{app:transition_probability_list_greedy_policy}, we define a
transition probability list $\mathcal{P}(\pi^t,\psi)$ for the policy $\pi^t$ and
write the reward function $r_{\pi^t}(\bm{\mu})$ with
$\mathcal{P}(\pi^t,\emptyset)$.

\textbf{Definition.} Assume the initial partial realization is
$\psi$. Let $\bm{s}_0$ be the status corresponding to $\psi$. Recall that
$\bm{s}_0 = (s_{1, 0}, \dots, s_{m, 0}) = (1 - \mu_1c_1(\psi), \dots, 1 -
\mu_mc_m(\psi))$. At the first step, policy $\pi^t$ chooses community $i^*_0 =
\argmax_{i\in [m]} 1 - c_i(\psi)\ubar{\mu}_{i,t}$. With probability $q^{\pi^t}_0
\defeq c_{i^*_0}(\psi)\mu_{i^*_0}$, the communities stay at the same status
$\bm{s}_0$. With probability $p^{\pi^t}_0 \defeq 1 -
c_{i^*_0}(\psi)\mu_{i^*_0}$, the communities transit to next status
$\bm{s}_1\defeq \bm{s}_0 -\mu_{i^*_0}\bm{I}$. Note that
\begin{equation*}
1 - c_i(\psi)\ubar{\mu}_{i,t} = \frac{\mu_i - (1 - s_{i,0})\ubar{\mu}_{i,t}}{\mu_i} = \frac{\ubar{\mu}_{i,t}}{\mu_i}s_{i,0}  + \frac{\mu_i - \ubar{\mu}_{i,t}}{\mu_i}. 
\end{equation*}
We recursively define $\bm{s}_{k+1}$ as $\bm{s}_{k} - \mu_{i^*_k}\bm{I}_{i^*_k}$
where $i^*_k \in \max_{i\in [m]} {(\ubar{\mu}_{i,t}/\mu_i)s_{i,k}} + (\mu_i -
\ubar{\mu}_{i,t})/\mu_i$. The transition probability $p^{\pi^t}_{k} \defeq
s_{i^*_k, k}$. We define the transition probability list $\mathcal{P}(\pi^t,
\psi) = (p^{\pi^t}_0, \dots, p^{\pi^t}_D)$ where $D = \sum_{i=1}^{m}(d_i -
c_i(\psi))$ is the number of distinct member we haven't meet under the partial
realization $\psi$. Note that it is possible that $p^{\pi^t}_{k}= 0$. In this
case, there is already no unmet members in $i^*_k$. The communities will be
stuck in status $\bm{s}_{k}$ since the policy $\pi^t$ always chooses community
$i^*_k$ to explore after the communities reach status $\bm{s}_{k}$. Hence, if
$k$ is the smallest index such that $p^{\pi^t}_{k} = 0$, we define
$p^{\pi^t}_{k^{\prime}} = 0$ for all $k^{\prime} > k$.

\textbf{Compute $\mathcal{P}(\pi^t, \psi)$}.
Define $\mathcal{B}_i(\psi) = \{1 - c_i(\psi)\mu_i, 1 - (1 + c_i(\psi))\mu_i,
\dots, \mu_i, 0\}$ for $i\in [m]$. Let $b_i\in \mathcal{B}_i(\psi), b_j\in
\mathcal{B}_j(\psi), i, j\in [m]$. We define a {\em sorting comparator} as
follows.
\begin{equation*}
{\rm less}(b_i, b_j) =
\bOne\{(\ubar{\mu}_{i,t}/\mu_i)\cdot b_i  + (\mu_i -
\ubar{\mu}_{i,t})/\mu_i < (\ubar{\mu}_{j,t}/\mu_i)\cdot b_{j}  + (\mu_j - \ubar{\mu}_{j,t})/\mu_j\}  
\end{equation*}
If $b_i \geq b_j$ and ${\rm less}(b_i, b_j) = 1$, we can infer that
$\ubar{\mu}_{i, t}/\mu_i \geq \ubar{\mu}_{j,t}/\mu_j$, which means the size of
community $j$ is more overestimated than the size of community $i$. The
overestimation leads to wrong order between $b_i$ and $b_j$ when using the
comparator ${\rm less}$. The list $\mathcal{P}(\pi^t, \psi)$ can be computed as
follows. Firstly, we sort elements in $\cup_{i\in [m]}\mathcal{B}_i$ with the
comparator ${\rm less}$. Secondly, we truncate the sorted list at the first zero
elements. Thirdly, we paddle zeros at the end of list until the length is $D +
1$. All the arguments in
Section~\ref{app:expected_reward_greedy_policy}-\ref{app:optimality_greedy_policy}
about $\mathcal{P}(\pi^g, \psi)$ can be easily extended to $\mathcal{P}(\pi^t,
\psi)$.

\textbf{Expected reward.} In the following, we still use the extended definition of reward
\begin{equation*}
  R(\bm{k}, \phi) =f\left(\sum_{i=1}^{m}\abs{\bigcup_{\tau=1}^{k_i}\{\phi(i,\tau)\}}\right),
\end{equation*}
where $f$ is a non-decreasing function. We can write the reward function
$r_{\pi^t}(\bm{\mu})$ as
\begin{equation*}
 r_{\pi^t}(\bm{\mu}) = \sum_{j=0}^{\min\{K, \sum_{i=1}^m d_i\}} f(j)\times p^{\pi^t}_0\times\dots\times p^{\pi^t}_{j-1}\times L(\{q^{\pi^t}_0, \dots, q^{\pi^t}_{j}\}, K - j). 
\end{equation*}
Here $p^{\pi^t}_j$ is element in $\mathcal{P}(\pi^t, \emptyset)$,
$q^{\pi^t}_j\defeq 1 - p^{\pi^t}_j$, and $K$ is the budget.

\subsection{Proof framework}\label{app:proof_framework_for_adaptive_exploration}
\textbf{Notations.}
Let $D = \sum_{i=1}^{m}d_i$ in this part. Let $\mathcal{P}(\pi^g, \emptyset) =
(p^{\pi^g}_0, \dots, p^{\pi^g}_{D})$ and $\mathcal{P}(\pi^t, \emptyset) =
(p^{\pi^t}_0, \dots, p^{\pi^t}_{D})$. According to
Corollary~\ref{corollary:observation_on_probability_list}, we know that
$\mathcal{P}(\pi^g, \emptyset)$ can be obtained by sorting $\cup_{i\in [m]}\{1,
1 - \mu_i, 1 - 2\mu_i, \dots, \mu_i\}\cup\{0\}$. Here we define another list
$\tilde{\mathcal{P}}(\pi^g)$ which is obtained by sorting $\cup_{i\in [m]}\{(i,
1), (i, 1 - \mu_i), \dots, (i, \mu_i)\}$ via comparing the second value in the
pair. Let $U_{i,k}$ denote how many times pair $(i, \cdot)$ appears in the first
$k$ positions in the list $\tilde{\mathcal{P}}(\pi^g)$. The value $U_{i,k}$
satisfies that $p^{\pi^g}_k = \max_{i=1}^{m}1 - U_{i,k}\mu_i$. Note that the
definition of $U_{i,k}$ are equivalent to the one defined in the main text.

\regretboundc*
\begin{proof}
When $\ubar{\bm{\mu}}_t$
is close to $\bm{\mu}$, the list $\mathcal{P}(\pi^t, \emptyset)$ is similar to
the list $\mathcal{P}(\pi^g, \emptyset)$, which indicates the reward gap
$r_{\pi^g}(\bm{\mu}) - r_{\pi^t}(\bm{\mu})$ is small. Let
$\bOne_{i, k}(\ubar{\bm{\mu}}_t)$ be the indicator that takes value 1 when
$\mathcal{P}(\pi^g, \emptyset)$ and $\mathcal{P}(\pi^t, \emptyset)$ are the same
for the first $k$ elements, and different at the $(k + 1)$-th elements (i.e.,
$p^{\pi^g}_j = p^{\pi^t}_j$ for $0\leq j \leq k-1$ and $p^{\pi^g}_k\neq
p^{\pi^t}_k$) with condition $p^{\pi^t}_k = 1 - U_{i, k}\mu_i$.
Note that the first $m$ elements in $\mathcal{P}(\pi^t, \emptyset)$ and
$\mathcal{P}(\pi^g, \emptyset)$ equal to 1. Then the reward gap at round $t$ is
\begin{equation*}
 \Delta_{\pi^t} = r_{\pi^g}(\bm{\mu}) - r_{\pi^t}(\bm{\mu}) = \sum_{i=1}^m\sum_{k=m+1}^{\min\{K, D\}} \bOne_{i,k}(\ubar{\bm{\mu}}_t)\cdot \Delta^{i, k}_{\max}, 
\end{equation*}
where $\Delta^{i, k}_{\max}$ is the maximum reward gap among all possible
$\ubar{\bm{\mu}}_t$ such that $\bOne_{i,k}(\ubar{\bm{\mu}}_t) = 1$, i.e.,
\begin{equation*}
 \Delta^{i, k}_{\max} = \max_{\forall \ubar{\bm{\mu}}_t, \bOne_{i,k}(\ubar{\bm{\mu}}_t) = 1} r_{\pi^{g}}(\bm{\mu}) - r_{\pi^t}(\bm{\mu}).
\end{equation*}
Note that 
\begin{equation*}
 \Delta^{i,k}_{\max} \leq \sum_{j = k}^{\min\{K, D\}} f(j) \times p^{\pi^g}_0 \times \cdots\times p^{\pi^g}_{j-1} \times L(\{1 - p^{\pi^g}_0, \cdots, 1 - p^{\pi^g}_j\}, K - j).
\end{equation*}

The expected cumulative regret can be expanded as
\begin{align*}
  \text{Reg}_{\bm{\mu}}(T) = \mathbb{E}_{\Phi_1,\cdots, \Phi_{T}}\left[\sum_{t=1}^{T}\Delta_{\pi^t}\right] &\leq \sum_{t=1}^T\mathbb{E}_{\Phi_1,\cdots, \Phi_{t-1}}\left[\sum_{k = {m + 1}}^{\min\{K, D\}} \sum_{i=1}^{m} \bOne_{i, k}(\ubar{\bm{\mu}}_t) \times \Delta^{i,k}_{\max}\right]\\
& = \sum_{i=1}^m \sum_{k = {m + 1}}^M \Delta^{i,k}_{\min}\mathbb{E}_{\Phi_1,\cdots, \Phi_{t-1}}\left[ \sum_{t=1}^T \mathbbm{1}_{i, k}(\ubar{\bm{\mu}}_t)\right].
\end{align*}
Our next step is bound $\mathbb{E}_{\Phi_1,\cdots, \Phi_{t-1}}\left[
 \sum_{t=1}^T \bOne_{i, k}(\ubar{\bm{\mu}}_t)\right]$. 
We rewrite the indicator
$\bOne_{i, k}(\ubar{\bm{\mu}}_t)$ as:
\begin{equation*}
 \bOne_{i, k}(\ubar{\bm{\mu}}_t) =  \bOne_{i, k}(\ubar{\bm{\mu}}_t)\bOne\{T_{i, t-1}\leq l_{i, k}\} + \mathbbm{1}_{i, k}(\ubar{\bm{\mu}}_t)\bOne\{T_{i, t-1}> l_{i, k}\},
\end{equation*}
where $l_{i, k}$ is a problem-specific constant. In
Lemma~\ref{lemma:enough_probing_bound}, we show that the probability we choose a
wrong community when community $i$ is probed enough times (i.e., $T_{i, t-1} >
l_{i, k}$) is very small. Based on the lemma, the regret corresponding to the
event $\bOne\left\{T_{i, t-1} > l_{i, k}\right\}$ is bounded as follows.
\begin{align*}
  & \sum_{i=1}^{m}\sum_{k=m + 1}^{\min\{K, D\}} \Delta^{i,k}_{\min} \mathbb{E}_{\Phi_1,\cdots, \Phi_T}\left[\sum_{t=1}^T\bOne_{i, k}(\ubar{\bm{\mu}}_t) \bOne\left\{ T_{i, t-1}  > l_{i,k}\right\}\right]  \leq \frac{\floor*{\frac{K^{\prime}}{2}}\pi^2}{3}\sum_{i=1}^{m}\sum_{k = m + 1}^{\min\{K, D\}}\Delta^{i, k}_{\max}.
\end{align*}
On the other hand, the regret associated with the event $\mathbbm{1}\{T_{i, t-1}
\leq l_{i, k}\}$ is trivially bounded by $\sum_{i=1}^{m}\sum_{k = m +
  1}^K\Delta^{i, k}_{\max}l_{i, k}$.
In conclusion, the expected cumulative regret is bound as
\begin{equation*}
  \begin{split}
 &\text{Reg}_{\bm{\mu}}(T) \leq  \sum_{i=1}^{m} \sum_{k=m + 1}^K \Delta^{i,k}_{\max}\mathbb{E}_{\Phi_1,\cdots, \Phi_T}\left[\sum_{t=1}^T\bOne_{k, t}(\ubar{\bm{\mu}}_t) \right]\\
 & \leq \sum_{i=1}^{m}\sum_{k = m + 1}^K\Delta^{i, k}_{\max}l_{i, k} + 
 \frac{\floor{\frac{K^{\prime}}{2}}\pi^2}{3}\sum_{i=1}^{m}\sum_{k = m + 1}^{\min\{K, D\}}\Delta^{i, k}_{\max}\\
  & \leq \left(  \sum_{i=1}^{m}\sum_{k = m + 1}^K \frac{6\Delta^{i, k}_{\max}}{(\Delta^{i,k}_{\min})^2}\right)\ln T + 
 \frac{\floor{\frac{K^{\prime}}{2}}\pi^2}{3}\sum_{i=1}^{m}\sum_{k = m + 1}^{\min\{K, D\}}\Delta^{i, k}_{\max}.
  \end{split}
\end{equation*}
Note $\Delta^{(k)}_{\max} \geq \max_{i\in [m]} \Delta^{i,k}_{\max}$. This completes the proof.
\end{proof}

\begin{lemma}\label{lemma:enough_probing_bound}
  For all $k\leq \{M, \sum_{i=1}^{m}d_i\}$, we have
  \begin{equation}
    \label{eq:enough_probing_bound}
    \mathbb{E}_{\Phi_1, \dots, \Phi_T}\left[  \sum_{t=1}^{T} \bOne_{i, k}(\ubar{\bm{\mu}}_t)\bOne\{T_{i,t-1} > l_{i, k}\} \right]  \leq \frac{\floor{\frac{K^{\prime}}{2}}\pi^2}{3},
  \end{equation}
  where $l_{i, k}$  is defined as $l_{i, k} \defeq 6\ln T/(\Delta^{i,k}_{\min})^2$.
\end{lemma}
\begin{proof}
  The following proof is similar to the that for the traditional Upper
  Confidence Bound (UCB) algorithm~\cite{auer2002finite}. In the following, we
  define $i^*_k = \max_{i\in [m]} 1 - U_{i, k}\mu_i$.
\begin{align*}
    &\sum_{t=1}^{T} \bOne_{i, k}(\ubar{\bm{\mu}}_t)\bOne\{T_{i,t-1} > l_{i, k}\}  = \sum_{t=l_{i, k}+1}^{T} \bOne_{i, k}(\ubar{\bm{\mu}}_t)\mathbbm{1}\{T_{i,t-1} > l_{i, k}\} \\
  \leq &\sum_{t = l_{i,k} + 1}^{T} \mathbbm{1}\{(\hat{\mu}_{i,t-1} - \rho_{i,t-1})U_{i,k}) < (\hat{\mu}_{i^*_k, t-1} - \rho_{i^*_k, t-1})U_{i^*_k, k}, T_{i,t-1} > l_{i, k}\}.\\
\end{align*}

When $T_{i_k, t-1} > l_{i, k} \triangleq \frac{6\ln T}{(\Delta^{i,k}_{\min})^2}$, we have
\begin{equation*}
  \rho_{i, t-1} = \sqrt{\frac{3\ln t}{2T_{i, t-1}}} < \frac{\Delta^{i,k}_{\min}}{2}\Rightarrow  \underbrace{\mu_{i^*_k}U_{i^*_k, k}< (\mu_{i} - 2\rho_{i, t-1})U_{i, k}}_{\text{$i$ and $i^*_k$ are distinguishable with high prob.}}.
\end{equation*}
If $i\neq i^*_k$ exists such that
\begin{align*}
  \hat{\mu}_{i^*_k,t-1} - \rho_{i^*_k,t-1} < \mu_{i^*_k}, \text{ and }\hat{\mu}_{i,t-1} + \rho_{i,t-1} > \mu_{i},
\end{align*}
we have
\begin{equation*}
  \begin{split}
   (\hat{\mu}_{i^*_k,t-1} - \rho_{i^*_k,t-1})U_{i^*_k, k}< \mu_{i^*_k}U_{i^*_k, k}  < (\mu_{i} - 2\rho_{i, t-1})U_{i, k} < (\hat{\mu}_{i,t-1} - \rho_{i,t-1})U_{i,k},\\
  \end{split}
\end{equation*}
which contradicts with $(\hat{\mu}_{i,t-1} - \rho_{i,t-1})U_{i,k} <
(\hat{\mu}_{i^*_k, t-1} - \rho_{i^*_k, t-1})U_{i^*_k, k}$.   
Hence when $T_{i, t-1} > l_{i, k}$, we have
\begin{equation*}
  \begin{split}
    &\left\{ (\hat{\mu}_{i,t-1} - \rho_{i,t-1})U_{i,k} < (\hat{\mu}_{i^*_k, t-1} - \rho_{i^*_k, t-1})U_{i^*_k, k}\right\} \\
    & \subseteq \left\{ \hat{\mu}_{i,t-1} + \rho_{i,t-1}
      \leq\mu_i \text{ or } \hat{\mu}_{i^*_k,t-1} - \rho_{i^*_k,t-1} \geq \mu_{i^*_k}\right\}
  \end{split}
\end{equation*}
Using the union bound, we have
\begin{equation*}
  \begin{split}
  &\Pr\left( (\hat{\mu}_{i,t-1} - \rho_{i,t-1})U_{i, k} < (\hat{\mu}_{i^*_k, t-1} - \rho_{i^*_k, t-1})U_{i^*_k, k} \right)\\
 \leq&\Pr\left( \hat{\mu}_{i,t-1} + \rho_{i,t-1}
   \leq\mu_i \text{ or } \hat{\mu}_{i^*_k,t-1} - \rho_{i^*_k,t-1} \geq \mu_{i^*_k} \right)\\
 \leq &\Pr\left( \hat{\mu}_{i,t-1} + \rho_{i,t-1} \leq\mu_i  \right) + \Pr\left( \hat{\mu}_{i^*_k,t-1} - \rho_{i^*_k,t-1} \geq \mu_{i^*_k} \right). \\
  \end{split}
\end{equation*}
Therefore, we can conclude that
\begin{align*}
    &\mathbb{E}_{\Phi_1, \dots, \Phi_T}\left[  \sum_{t=1}^{T} \mathbbm{1}_{i, k}(\ubar{\bm{\mu}}_t)\mathbbm{1}\{T_{i,t-1} > l_{i, k}\}\right] \\
  \leq &\sum\nolimits_{t = l_{i,k} + 1}^{T} \bOne\{(\hat{\mu}_{i,t-1} - \rho_{i,t-1})U_{i, k} < (\hat{\mu}_{i^*_k, t-1} - \rho_{i^*_k, t-1})U_{i^*_k, k}, T_{i,t-1} > l_{i, k}\}.\\
  \leq &\sum\nolimits_{t = l_{i,k} + 1}^{T}  \Pr\left\{\hat{\mu}_{i,t-1} + \rho_{i,t-1} \leq\mu_i \right\} +
  \Pr \left\{  \hat{\mu}_{i^*_k,t-1} - \rho_{i^*_k,t-1} \geq \mu_{i^*_k} \right\} \\
  \leq&\sum_{t = l_{i,k} + 1}^{T}\left(  \sum\nolimits_{T_{i, t-1} = l_{i, k} + 1} ^{t\floor*{\frac{K^{\prime}}{2}}}\Pr\left\{ \hat{\mu}_{i,t-1} + \rho_{i,t-1} \leq\mu_i | T_{i, t-1} \right\}\right.\\
  & \left.+ \sum\nolimits_{T_{i^*_k, t-1}  = 1} ^{t\floor*{\frac{K^{\prime}}{2}}}\Pr\left( \hat{\mu}_{i^*_k,t-1} - \rho_{i^*_k,t-1} \geq\mu_{i^*_k} | T_{i^*_k, t-1}\right)\right)\\
    & \leq \sum_{t=1}^{\infty} 2t\floor*{\frac{K^{\prime}}{2}} \times t^{-3} = 2\floor*{\frac{K^{\prime}}{2}}\sum_{t=1}^{\infty}t^{-2} =\frac{\floor*{\frac{K^{\prime}}{2}}\pi^2}{3}.\qedhere
\end{align*}
\end{proof}

\subsection{Full information feedback}
If we feed the empirical mean in the exploration oracle, then the policy $\pi^t$
is determined by $\hat{\bm{\mu}}_t$. Similarly, we can define the event
$\bOne_{i,k}(\hat{\bm{\mu}}_t)$ by replacing $\ubar{\bm{\mu}}_t$ with
$\hat{\bm{\mu}}$ in
Section~\ref{app:transition_probability_list_pit}-\ref{app:proof_framework_for_adaptive_exploration}.

\begin{lemma}\label{lemma:full_information_feedback_empirical_mean}
  If we make revisions defined in
  Eq.~\eqref{eq:revision_2} to Algo.~\ref{algo:CLCB_algorithm} and
  feed the empirical mean in \texttt{\textproc{CommunityExplore}} to
  explore communities adaptively,
  then for all community $C_i$ and $k\leq \{K, \sum_{i=1}^m d_i\}$, we have
  \begin{equation}
    \label{eq:full_information_feedback_empirical_mean}
    \mathbb{E}_{\Phi_1, \dots, \Phi_T}\left[ \sum_{t=2}^T\bOne_{i, k}(\hat{\bm{\mu}}_t) \right] \leq\frac{2}{\varepsilon^4_{i, k}},  
  \end{equation}
  where $\varepsilon_{i, k}$ is defined as (here $i^*_k\in \argmin_{i\in [m]} \mu_iU_{i,k}$)
  \begin{equation*}
    \varepsilon_{i, k} \triangleq \frac{\mu_iU_{i,k} - \mu_{i^*_k}U_{i^*_k, k}}{U_{i,k} + U_{i^*_k, k}} \text{ for } i\neq i^*_k \text{ and } \varepsilon_{i, k} = \infty \text{ for } i =  i^*_k.
  \end{equation*}
\end{lemma}
\begin{equation*}
\end{equation*}
\begin{proof}
We first bound the probability of the following event by relating
$\mathbbm{1}_{i, k}(\hat{\bm{\mu}}_t)$ with the event that both $\mu_{i, t-1}$
and $\mu_{i_k, t-1}$ in the confidence interval $\varepsilon_{i,k}$.
\begin{equation*}
    \mathbbm{1}_{i,k}(\hat{\bm{\mu}}_t) \leq \mathbbm{1} \left\{\hat{\mu}_{i, t-1}U_{i,k}  < \hat{\mu}_{i^*_k, t-1}U_{i^*_k,k} \right\}.
\end{equation*}
If $i\neq i^*_k$ such that
\begin{align*}
  \hat{\mu}_{i, t-1} > \mu_i - \varepsilon_{i, k}, \text{ and }\hat{\mu}_{i^*_k, t-1}  < \mu_{i^*_k} + \varepsilon_{i, k},
\end{align*}
then
\begin{align*}
  \hat{\mu}_{i, t-1}U_{i,k} &> (\mu_i - \varepsilon_{i, k}) U_{i,k} = (\mu_{i^*_k} + \varepsilon_{i, k})U_{i^*_k, k} > \hat{\mu}_{i^*_k, t-1}U_{i^*_k, k},
\end{align*}
which contradicts with that $\hat{\mu}_{i, t-1}U_{i,k} <
\hat{\mu}_{i^*_k, t-1}U_{i^*_k}$.
Here $(\mu_i - \varepsilon_{i, k}) U_{i,k} = (\mu_{i^*_k} + \varepsilon_{i^*,
k})U_{i^*_k, k}$ can be derived from the definition of $\varepsilon_{i,k}$.
Therefore
\begin{align*}
  &\mathbbm{1}\left\{\hat{\mu}_{i, t-1}U_{i,k}  < \hat{\mu}_{i^*_k, t-1}U_{i^*_k, k}\right\}\\
 \leq\ &\mathbbm{1}\left\{\hat{\mu}_{i, t-1} \leq \mu_i - \varepsilon_{i, k} \text{ or } \hat{\mu}_{i^*_k, t-1} \geq \mu_{i^*_k} + \varepsilon_{i, k}\right\}.
\end{align*}
With above equation and the concentration bound in~\cite{Dubhashi2009CMA}, the expectation $\mathbb{E}_{\Phi_1, \dots, \Phi_T}\left[ \sum_{t=2}^{T}\bOne_{i,k}(\hat{\bm{\mu}}_t)\right]$ can be bounded as
\begin{align*}
  &\mathbb{E}_{\Phi_1, \dots, \Phi_T}\left[ \sum_{t=2}^T\bOne_{i, k}(\hat{\bm{\mu}}_t)\right]\\
  \leq &\sum_{t=2}^T\Pr\left\{\hat{\mu}_{i, t-1}U_{i,k}  < \hat{\mu}_{i^*_k, t-1}U_{i^*_k,k} \right\} \\
   \leq &\sum_{t=2}^T\Pr\left\{\hat{\mu}_{i, t-1} \leq \mu_i - \varepsilon_{i, k}\right\} + \Pr\left\{\hat{\mu}_{i^*_k, t-1} \geq \mu_{i^*_k} + \varepsilon_{i, k}\right\}\\
 \leq &\sum_{t=2}^T \left( \sum_{T_{i,t-1} = t-1}^{t\floor{K^{\prime}/2}} e^{-\varepsilon^2_{i, k}T_{i, t-1}} + \sum_{T_{i^*_k,t-1} = t-1}^{t\floor{K^{\prime} / 2}} e^{-\varepsilon^2_{i, k}T_{i^*_k, t-1}}\right)\\
 \leq &\ 2\sum_{t=1}^T \sum_{s=t}^{\infty} e^{-s\varepsilon^2_{i, k}} \leq 2\sum_{t=1}^T \frac{e^{-t\varepsilon^2_{i, k}}}{\varepsilon^2_{i, k}} \leq \frac{2}{\varepsilon^4_{i, k}}.\qedhere
\end{align*}
\end{proof}

\section{Experimental Evaluation}
In this section, we conduct simulations to validate the theoretical results
claimed in the main text and provide some insight for future research.
\subsection{Offline Problems}
In this part, we show some simulation results for the offline
problems.

\textbf{Performance of Algorithm~\ref{algo:non_adaptive_exploration}}. In Fig.~\ref{fig:allocation_lower_upper_bound}, we show that the allocation
lower bound $\bm{k}^{-}$ and upper bound $\bm{k}^{+}$ are close to the optimal
budget allocation. From Fig.~\ref{fig:allocation_lower_upper_bound}, we observe
that the $L1$ distance between $\bm{k}^{*}$ and $\bm{k}^{-}$ (or $\bm{k}^{+}$)
is around $m/2$, which means the average time complexity of
Algorithm~\ref{algo:non_adaptive_exploration} is $\Theta((m\log m) / 2)$.

\begin{figure}[t]
  \centering
  \includegraphics[scale = 0.8]{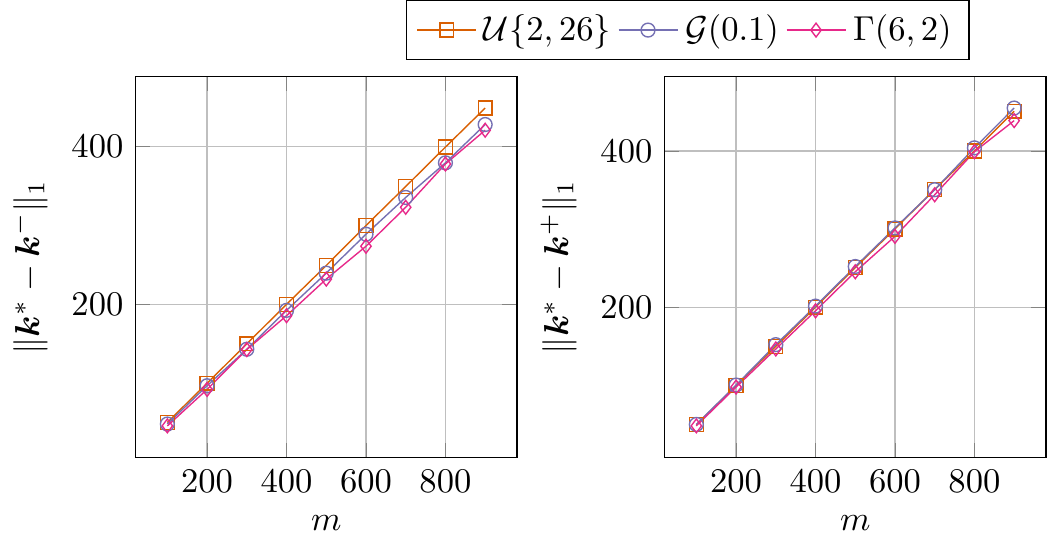}
  \caption{The $L1$ distance between $\bm{k}^{*}$ and $\bm{k}^{-}$,
    $\bm{k}^{+}$ under different community size distributions.
    Here $\mathcal{U}\{2, 26\}$ is the discrete uniform
    distribution between 2 and 26. $\mathcal{G}(0.1)$ is the geometric
    distribution with success probability $0.1$ on the support set $\{2, 3,
    \dots\}$. $\Gamma(\alpha, \beta)$ is the gamma distribution with shape
    $\alpha$ and rate $\beta$. We discretize
    the support set of the gamma distribution and add 2 to all the values in the
    support set to ensure that the minimum size of communities is 2. The budget
    $K$ is a random number between $m + 1$ and $\sum_{i}d_i$. We run the
    simulations for 1000 times for each data point.}
  \label{fig:allocation_lower_upper_bound}
\end{figure}

\textbf{Reward v.s. Budget}. We show the relationship between the reward (i.e.,
the number of distinct members) and the given budget in
Fig.~\ref{fig:offline_reward_budget}. From Fig.~\ref{fig:offline_reward_budget},
we can draw the following conclusions.
\begin{itemize}
\item The performance of the four methods are ranked as: \textit'{``Adaptive
    Opt.''}, \textit{``Non-adaptive Opt.''}, \textit{``Proportional to Size''},
  \textit{``Random Allocation''}. This validate our optimality results in
  Sec.~\ref{sec:offline}.
\item The method ``\textit{Proportional to Size}'' and ``\textit{Non-adaptive
    Opt.}'' have similar performance. It is an intuitive idea to allocate
  budgets proportional to the
  community sizes. The simulation results also demonstrate the efficiency of
  such budget allocation method. In the following, we analyze the reason
  theoretically. Recall the definition of $\bm{k}^{-}$ as follows.
 \begin{equation*}
   k^-_i = \frac{(K - m) / \ln (1 - \mu_i)}{\sum_{j=1}^m 1 / \ln (1 - \mu_j)}.
 \end{equation*}
 When $\mu_i\ll 1$, we have $\ln (1 - \mu_i)\approx -\mu_i$. Hence,
  \begin{equation*}
   k^-_i \approx \frac{(K - m) d_i}{\sum_{j=1}^m d_j}.
 \end{equation*}
 Besides, the L1 distance between $\bm{k}^*$ and $\bm{k}^{-}$ is smaller than
 $m$. We can conclude that the budget allocation proportional to size is close
 to the optimal budget allocation. Fig.~\ref{fig:offline_budget_allocation} also
 validates this conclusion.    
\item The reward gap between ``\textit{Non-adaptive Opt.}'' and
  ``\textit{Adaptive Opt.}'' increases first and then decreases, as shown in
  Fig.~\ref{fig:used_budget_greedy_policy}.
\end{itemize}

\begin{figure}[t]
  \centering
    \setcounter{subfigure}{0}
    \subfloat[][Community size distributions]{
      \includegraphics[scale=0.7]{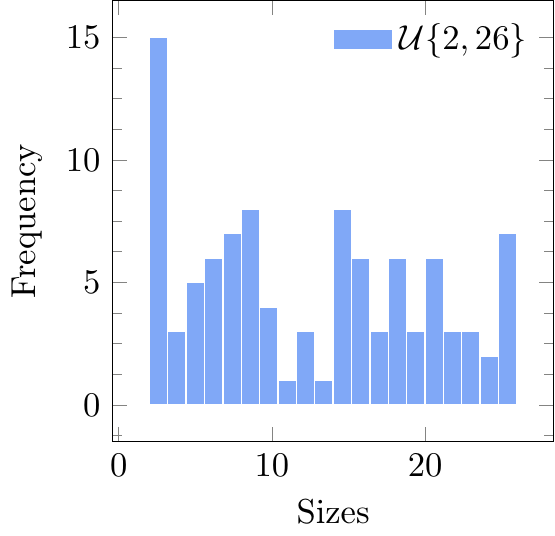}
      \includegraphics[scale=0.7]{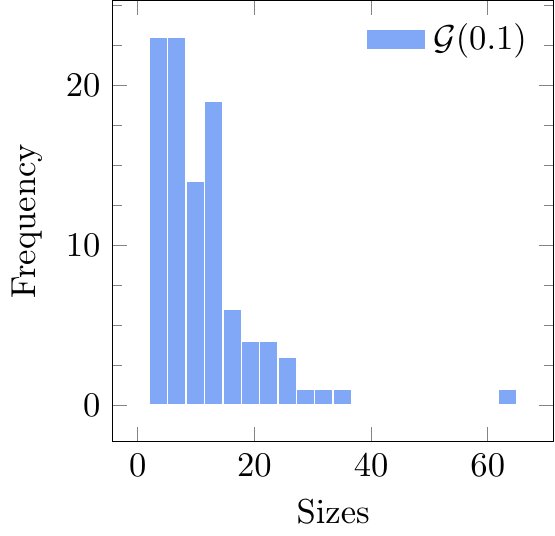}
      \includegraphics[scale=0.7]{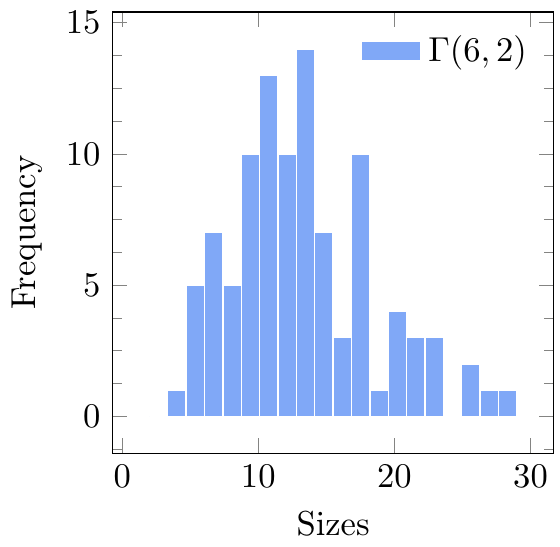}
    }\\
      \includegraphics[scale=0.85]{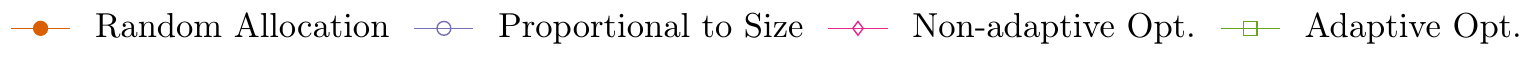}\\
    \subfloat[][Reward of different methods]{
      \includegraphics[scale=0.7]{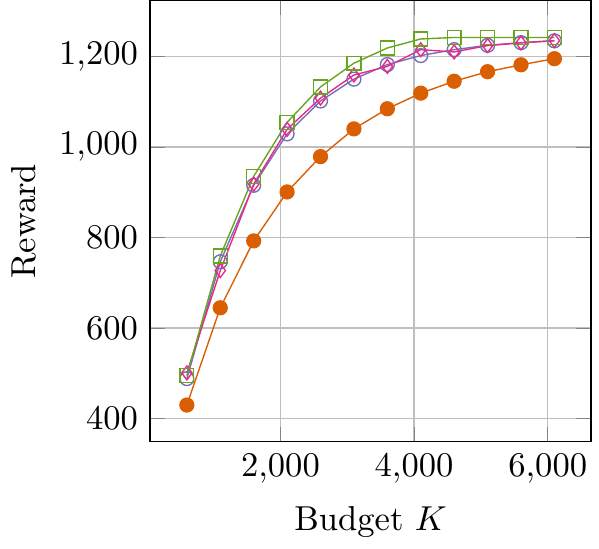}
      \includegraphics[scale=0.7]{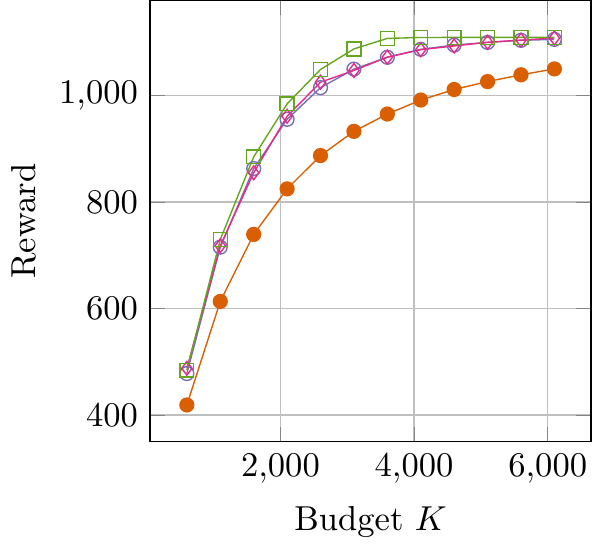}
      \includegraphics[scale=0.7]{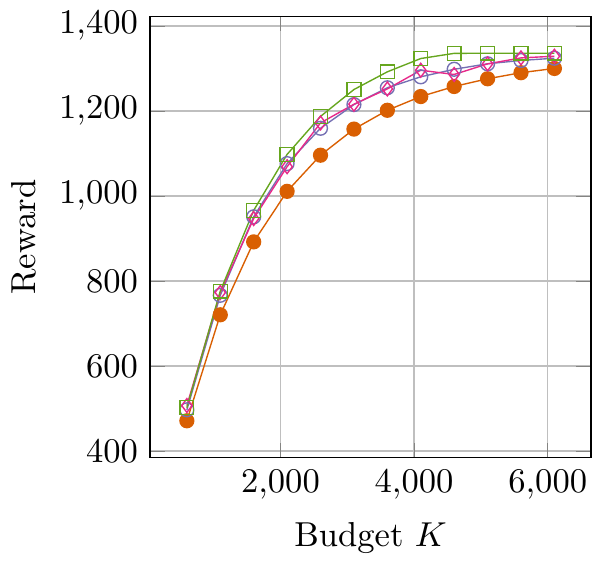}
    }
 \caption{Reward v.s. Budget. In the first row, we show three different size
   distributions of $m = 100$ communities. In the second row, we show
   the reward of four different budget allocation methods.
   Here \textit{``Random Allocation''}
   represents random budget allocation (sum up to $K$). \textit{``Proportional
     to Size''}
   method allocates budget proportional to the community sizes.
   \textit{``Non-adaptive Opt.''}
   corresponds to the optimal budget allocation obtained by the greedy method.
   \textit{``Adaptive Opt.''} means we explore the communities with greedy adaptive
   policy $\pi^g$. The simulations are run for 200 times for each data point on
   the budget-reward curve.}\label{fig:offline_reward_budget} 
\end{figure}

\textbf{Budget Allocation Comparison}. Fig.~\ref{fig:offline_budget_allocation}
and Fig~\ref{fig:used_budget_greedy_policy} show the budget allocation of
non-adaptive optimal method and adaptive optimal method.
Fig.~\ref{fig:used_budget_greedy_policy} shows that the adaptive optimal method
use the budget more efficiently.

\begin{figure}[t]
  \centering
  \includegraphics[scale=0.6]{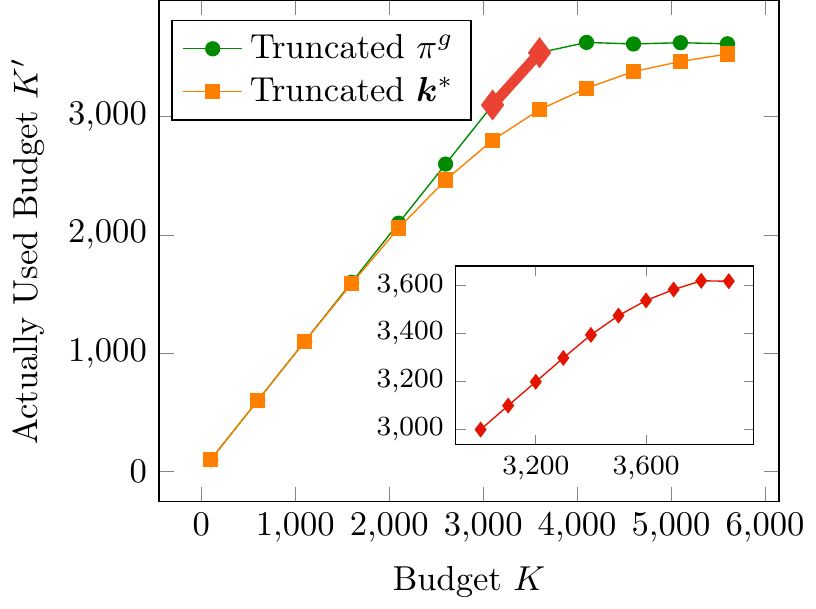}
  \includegraphics[scale=0.6]{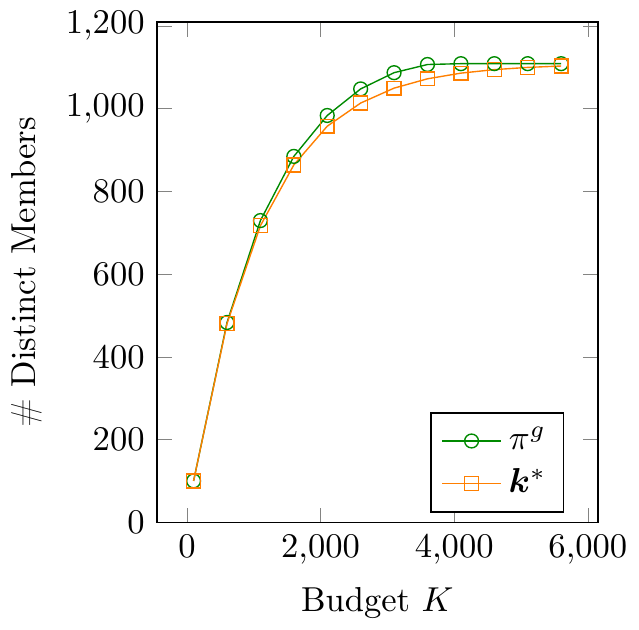}
  \includegraphics[scale=0.6]{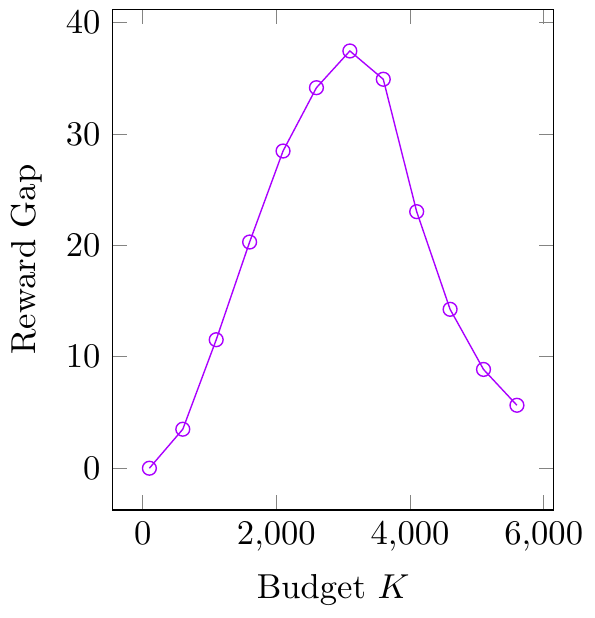}
  \caption{Actually used budget. we only show the results for the community size
    configuration generated by $\mathcal{G}(0.1)$, as shown in 
    the first row of Fig.~\protect\ref{fig:offline_reward_budget}.
    The legend labels have the same meaning as in Fig.~\protect\ref{fig:offline_budget_allocation}}
  \label{fig:used_budget_greedy_policy}
\end{figure}

\begin{figure}[t]
  \centering
    \setcounter{subfigure}{0}
      \includegraphics[scale=0.7]{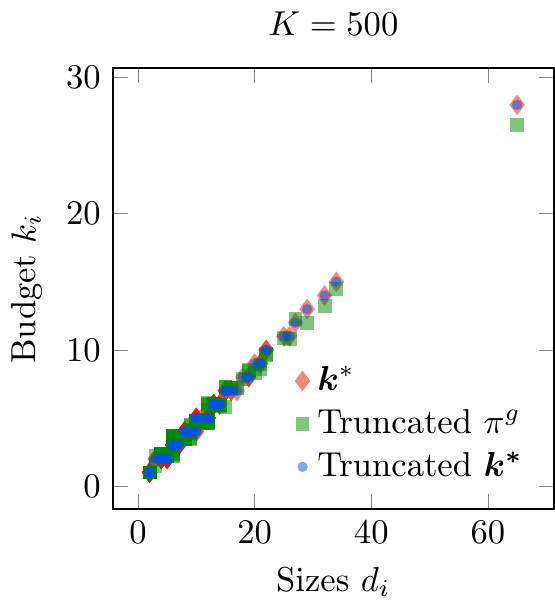}
      \includegraphics[scale=0.7]{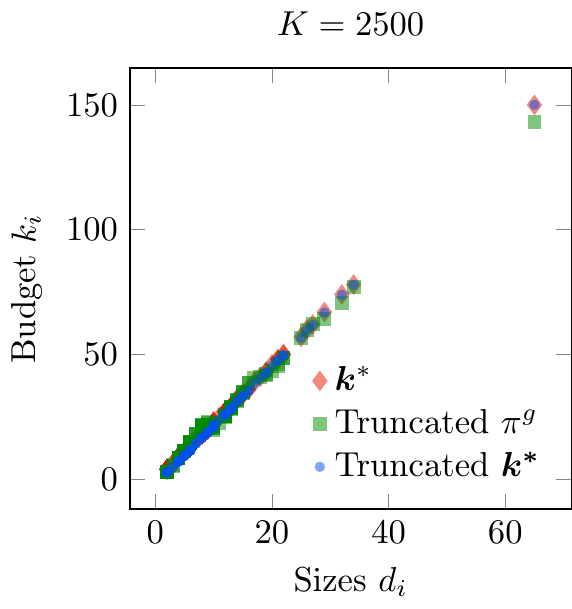}
      \includegraphics[scale=0.7]{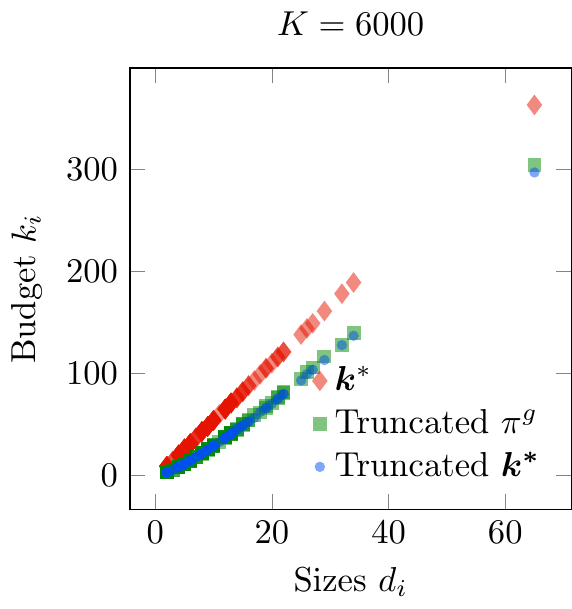}
         
 \caption{Comparison of different budget allocation methods. The distribution of
   community sizes generated by the geometric distribution with success
   probability $0.1$, as shown in the first row of
   Fig.~\protect\ref{fig:offline_reward_budget}. The legend label
   ``$\bm{k}^*$'' represents the optimal budget allocation. The ``truncated
   $\pi^g$'' means we stop the greedy adaptive process if all the members are
   found. The ``truncated $\bm{k}^*$'' means we stop the non-adaptive exploration
   of community $C_i$ if all the members of $C_i$ are found. Each data point is an
   average of 1000 simulations. }\label{fig:offline_budget_allocation}
\end{figure}

\subsection{Online Problems}
In the following, we show the simulation results for the online, non-adaptive
problem. The simulation results for online, adaptive are similar. Hence, we only
present the results for online, non-adaptive problems.
Fig.~\ref{fig:non_adaptive_toy_regret} shows the regret of three different
learning methods. For illustration purpose, we set the community sizes as
$\\bm{d} = (2, 3, 5, 6, 8, 10)$. From Fig.~\ref{fig:non_adaptive_toy_regret}, we
can draw the following conclusions.
\begin{itemize}
\item If we feed the empirical mean into the oracle, the regret grows
  linearly.  
\item The regret of CLCB algorithm is bounded logarithmically, as proved in Thm.~\ref{thm:regret_bound_non_adaptive_exploration_1}.
\item The regret under full information feedback setting is bounded as a problem
  related constant, as proved in Thm.~\ref{thm:non_adaptive_full_information}. 
\end{itemize}

\begin{figure}[t]
  \centering
  \includegraphics{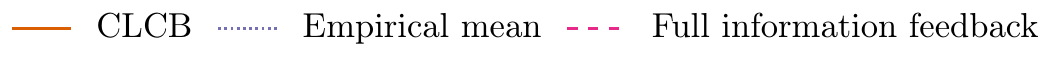}
    \setcounter{subfigure}{0}
    \subfloat[][$K = 20, \bm{k}^* = (1, 2, 3, 3, 5, 6)$]{
      \includegraphics[scale=0.8]{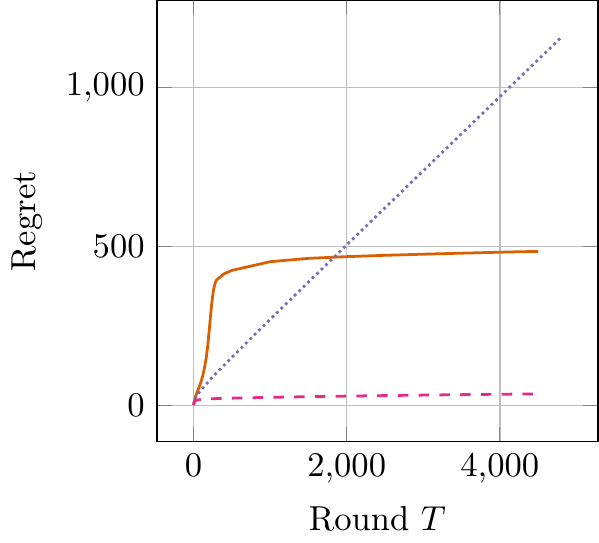}
      \includegraphics[scale=0.8]{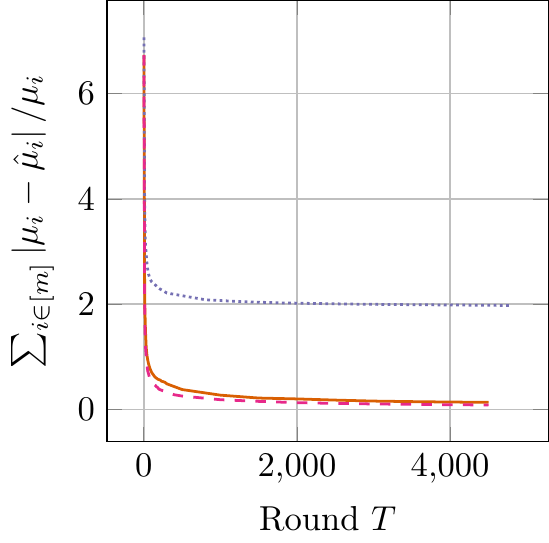}
    }\\
    \subfloat[][$K = 30, \bm{k}^* = (2, 3, 4, 5, 7, 9)$]{
      \includegraphics[scale=0.8]{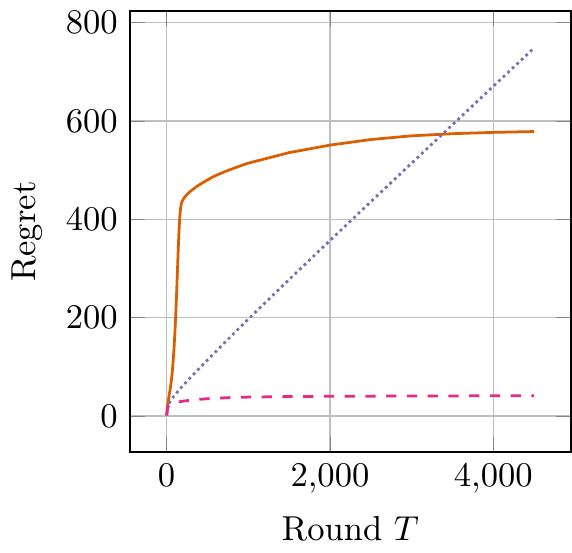}
      \includegraphics[scale=0.8]{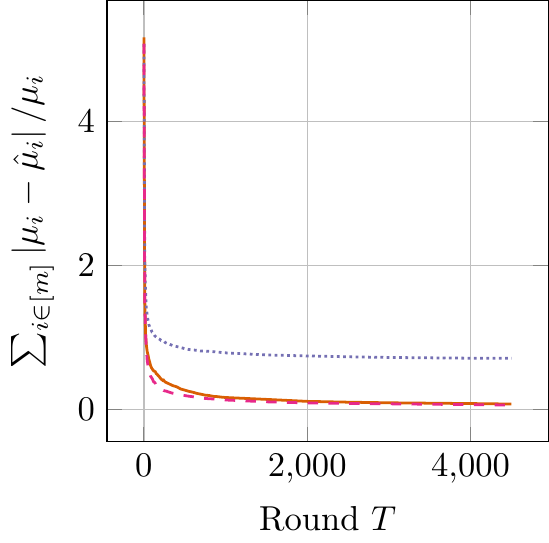}
    }\\
    \subfloat[][$K = 50, \bm{k}^* = (3,  4,  7,  9, 12, 15)$]{
      \includegraphics[scale=0.8]{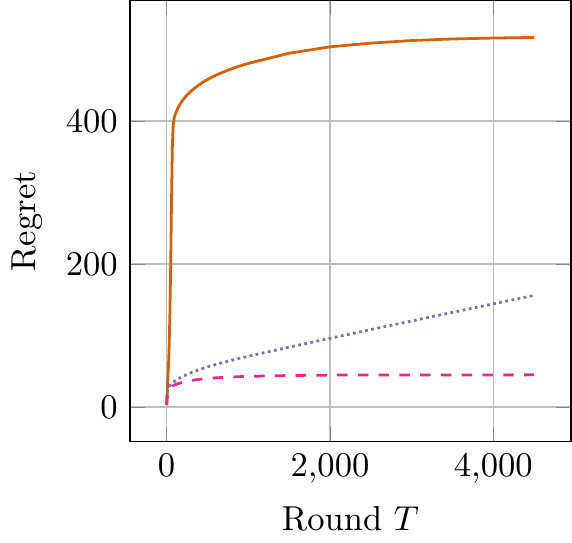}
      \includegraphics[scale=0.8]{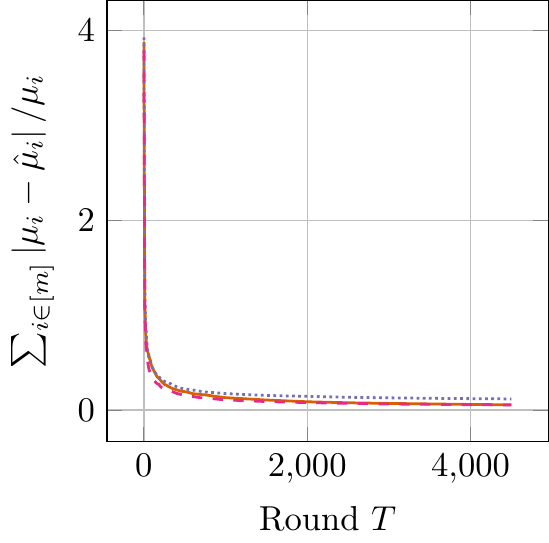}
    }
 \caption{Comparison of different learning algorithms. The sizes of communities
   are $\bm{d} = (2, 3, 5, 6, 8, 10)$. Here the label \textit{Empirical mean}
   represents feeding the empirical mean into the oracle directly.
   The regret/error line plots are average
   of 100 simulations.}\label{fig:non_adaptive_toy_regret} 
\end{figure}

\end{appendices}

\end{document}